\documentclass{article}

\usepackage{amsmath, amsfonts, amsthm}
\usepackage{xcolor}

\usepackage{microtype}
\usepackage{graphicx}
\usepackage[skip=0pt]{caption}
\usepackage{subcaption}
\usepackage{wrapfig}
\usepackage{booktabs} 

\usepackage{hyperref}

\usepackage[accepted]{icml2019}

\usepackage[utf8]{inputenc} 
\usepackage[T1]{fontenc}    
\usepackage{url}            
\usepackage{nicefrac}       
\usepackage{bm}
\usepackage{dsfont}
\usepackage[inline]{enumitem}
\usepackage[super]{nth} 
\usepackage{xspace}
\usepackage[section]{placeins}
\allowdisplaybreaks 

\icmltitlerunning{First-order Adversarial Vulnerability of Neural Networks and Input Dimension}

\newcommand{\loss}{\mathcal L}
\newcommand{\lnorm}{\mathcal \ell}
\renewcommand{\P}{\mathcal P}
\newcommand{\dset}{d_\pp}
\newcommand{\Dset}{\mathcal D(x,o)}

\renewcommand{\H}{\ref{H}}
\renewcommand{\S}{\ref{S}}

\newcommand{\e}[1]{\mathop{\mathbb{E}}\left[#1\right]}
\newcommand{\E}[1]{\mathop{\mathbb{E}} \! #1}
\newcommand{\ee}[2]{\mathop{\mathbb{E}_{#1}}\left[#2\right]}
\newcommand{\EE}[2]{\mathop{\mathbb{E}_{#1}} \! #2}

\newcommand{\xx}{\bm x}
\newcommand{\dd}{\bm \delta}
\newcommand{\pp}{{\bm p}}
\newcommand{\qq}{{\bm q}}
\renewcommand{\aa}{{\bm a}}

\newcommand{\dl}{\bm {\partial_{\scriptscriptstyle x}} \loss}
\newcommand{\dli}{\partial_{\scriptscriptstyle{x}} \loss}
\newcommand{\dlk}{\partial_{\scriptscriptstyle{k}} \loss}
\newcommand{\dfk}{\bm{\partial_{\scriptscriptstyle{x}}} f_{\scriptscriptstyle k}}
\newcommand{\dfh}{\bm{\partial_{\scriptscriptstyle{x}}} f_{\scriptscriptstyle h}}
\newcommand{\dfc}{\bm{\partial_{\scriptscriptstyle{x}}} f_{\scriptscriptstyle c}}
\newcommand{\dfik}{\partial_{\scriptscriptstyle{x}} f_{\scriptscriptstyle k}}

\newcommand{\dfic}{\partial_{\scriptscriptstyle{x}} f_{\scriptscriptstyle c}}

\DeclareMathOperator{\sign}{\mathrm{sign}}

\newcommand{\norm}[1]{\left \| #1 \right \|}
\newcommand{\dn}[1]{{\left\vert\kern-0.25ex\left\vert\kern-0.25ex\left\vert #1 \right\vert\kern-0.25ex\right\vert\kern-0.25ex\right\vert}}
\newcommand{\fa}{\ref{fig:results}a\xspace}
\newcommand{\fb}{\ref{fig:results}b\xspace}
\newcommand{\fc}{\ref{fig:results}c\xspace}
\newcommand{\fd}{\ref{fig:results}d\xspace}
\newcommand{\fe}{\ref{fig:results}e\xspace}
\newcommand{\ff}{\ref{fig:results}f\xspace}

\newcommand{\ffa}{\ref{fig:bcifar_summary}a\xspace}
\newcommand{\ffb}{\ref{fig:bcifar_summary}b\xspace}
\newcommand{\ffc}{\ref{fig:bcifar_summary}c\xspace}
\newcommand{\ffd}{\ref{fig:bcifar_summary}d\xspace}

\newcommand{\fga}{\ref{fig:bcifar_all}a\xspace}
\newcommand{\fgb}{\ref{fig:bcifar_all}b\xspace}
\newcommand{\fgc}{\ref{fig:bcifar_all}c\xspace}

\newcommand{\fge}{\ref{fig:bcifar_all}e\xspace}
\newcommand{\fgf}{\ref{fig:bcifar_all}f\xspace}

\newtheorem{theorem}{Theorem}
\newtheorem{proposition}[theorem]{Proposition}
\newtheorem{corollary}[theorem]{Corollary}
\newtheorem{lemma}[theorem]{Lemma}

\theoremstyle{definition}
\newtheorem{definition}[theorem]{Definition}

\theoremstyle{remark}
\newtheorem{remark}{Remark}

\graphicspath{{figures/}}

\begin{document}

\twocolumn[
\icmltitle{First-order Adversarial Vulnerability of Neural Networks and Input Dimension}



\icmlsetsymbol{equal}{*}

\begin{icmlauthorlist}
\icmlauthor{Carl-Johann Simon-Gabriel}{mpi,fb}
\icmlauthor{Yann Ollivier}{fb}
\icmlauthor{Bernhard Sch\"olkopf}{mpi}
\icmlauthor{L\'eon Bottou}{fb}
\icmlauthor{David Lopez-Paz}{fb}
\end{icmlauthorlist}

\icmlaffiliation{fb}{Facebook AI Research, Paris/New York}
\icmlaffiliation{mpi}{Empirical Inference Department, Max Planck Institute for
Intelligent Systems, T\"ubingen, Germany}

\icmlcorrespondingauthor{Carl-Johann Simon-Gabriel}{cjsimon@tue.mpg.de}

\icmlkeywords{Adversarial Examples}

\vskip 0.3in
]



\printAffiliationsAndNotice{}  

\begin{abstract}
    Over the past few years, neural networks were proven vulnerable to
    adversarial images: targeted but imperceptible image perturbations lead to
    drastically different predictions. We show that adversarial vulnerability
    increases with the gradients of the training objective when viewed as a
    function of the inputs. Surprisingly, vulnerability does not depend on
    network topology: for many standard network architectures, we prove that at
    initialization, the $\lnorm_1$-norm of these gradients grows as the square
    root of the input dimension, leaving the networks increasingly vulnerable
    with growing image size. We empirically show that this dimension dependence
    persists after either usual or robust training, but gets attenuated with
    higher regularization.
\end{abstract}

\section{Introduction}

Following the work of \citet{goodfellow15explaining}, Convolutional Neural
Networks (CNNs) have been found vulnerable to adversarial examples: an
adversary can drive the performance of state-of-the art CNNs down to chance
level with imperceptible changes to the inputs.

Based on a simple linear model, \citeauthor{goodfellow15explaining}\
already noted that adversarial vulnerability should depend on input
dimension.  \Citet{gilmer18adversarial,shafahi19adversarial} later
confirmed this, by showing that adversarial robustness
is harder to obtain with larger input dimension. However, these results
are different in nature from \citeauthor{goodfellow15explaining}'s
original observation: they rely on assumptions on the dataset that amount
to a form of uniformity in distribution over the input dimensions (e.g.\
concentric spheres, or bounded densities with full support). In the end,
this analysis tends to incriminate the data: if the data can be anything, and in
particular if it can spread homogeneously across many
input dimensions, then robust classification gets harder.

Image datasets do not satisfy these assumptions: they do not have full
support, and their probability distributions get more and more peaked with
larger input dimension (pixel correlation increases).  Intuitively, for image
classification, higher resolution should help, not hurt. 
Hence data might be the wrong culprit: if we want to understand the
vulnerability of our classifiers, then we should understand what is wrong with
our classifiers, not with our images. 

We therefore follow \citeauthor{goodfellow15explaining}'s original
approach, which explains adversarial vulnerability by properties of the
classifiers. Our main theoretical results start by formally extending
their result for a single linear layer to almost all current deep
feedforward network architectures. There is a further correction: based
on the gradients of a linear layer, \citeauthor{goodfellow15explaining}
predicted a linear increase of adversarial vulnerability with input
dimension $d$. However, they did not take into account that a layer's typical weights decrease
like $\sqrt d$. Accounting for this, the dependence becomes $\sqrt d$
rather than $d$, which is confirmed by both our theory and experiments. 

Our approach relies on evaluating the norm of gradients of the network
output with respect to its inputs. At first order, adversarial
vulnerability is related to gradient norms. We show that this norm is a
function of input dimension only, whatever the network architecture is.
The analysis is fully formal at initialization, and experiments show that
the predictions remain valid throughout training with very good
precision.

Obviously, this approach assumes that the classifier and loss are
differentiable. So arguably it is unclear whether it can explain the
vulnerability of networks with obfuscated or masked gradients. Still,
\citet{athalye18obfuscated} recently showed that masked gradients only
give a false sense of security: by reconstructing gradient approximations
(using differentiable nets!), the authors circumvented all
state-of-the-art masked-gradient defenses. This suggests that explaining
the vulnerability of differentiable nets is crucial, even for
non-differentiable nets. 

Although adversarial vulnerability was known to increase with gradient norms,
the exact relation between the two, and the approximations made, are
seldom explained, let alone tested empirically.
Section~\ref{sec:av_and_gradients} therefore starts with a detailed
discussion of the relationship between adversarial vulnerability and
gradients of the loss. Precise definitions
help with sorting out all approximations used.  We also revisit and
formally link several old and recent defenses, such as
double-backpropagation \mbox{\citep{drucker91double}} and FGSM
\mbox{\citep{goodfellow15explaining}}.  Section~\ref{sec:grad_estimation}
proceeds with our main theoretical results on the dimension dependence of
adversarial damage. Section~\ref{experiments} tests our predictions
empirically, as well as the validity of all approximations.

Our contribution can be summarized as follows.
\begin{itemize}
\item We show an empirical one-to-one relationship between average gradient norms and
adversarial vulnerability. This confirms that an essential part of adversarial
vulnerability arises from first-order phenomena.
\item We formally prove that, at initialization, the first-order vulnerability of common
neural networks increases as $\sqrt d$ with input dimension $d$.
Surprisingly, this is almost
independent of the architecture. Almost all current architectures are hence, by
design, vulnerable at initialization.
\item We empirically show that this dimension dependence persists after both
usual and robust (PGD) training, but gets dampened and eventually vanishes with
higher regularization. Our experiments suggest that PGD-regularization
effectively recovers dimension independent accuracy-vulnerability trade-offs.
\item We observe that further training after the 
training loss
has reached its minimum can provide improved test accuracy, but severely damages the
network's robustness. The last few accuracy points require a considerable
increase of network gradients.
\item We notice a striking discrepancy between the gradient norms (and
therefore the vulnerability) on the training and test sets respectively. It
suggests that gradient properties do not generalize well and that, outside the
training set, networks may tend to revert to initialization-like gradient
properties.
\end{itemize}

Overall, our results show that, without strong regularization, the gradients
and vulnerability of current networks naturally tend to grow with input
dimension. This suggests that current networks have too many degrees of
`gradient-freedom'. Gradient regularization can counter-balance this
to some extent, but on the long run, our networks may benefit from
incorporating more data-specific knowledge. The independence of our results on
the network architecture (within the range of currently common
architectures) suggests that doing so would require new network
modules.

\paragraph{Related Literature} \Citet{goodfellow15explaining} already noticed
the dimension dependence of adversarial vulnerability. As opposed to
\citet{amsaleg17vulnerability,gilmer18adversarial,shafahi19adversarial}, their
(and our) explanation of the dimension dependence is data-independent.
Incidentally, they also link adversarial vulnerability to loss gradients and
use it to derive the FGSM adversarial augmentation defense (see
Section~\ref{sec:av_and_gradients}). \Citet{ross18improving} propose to
robustify networks using the old double-backpropagation, but make no connection
to FGSM and adversarial augmentation (see our Prop.\ref{lossEqui}).
\Citet{lyu15unified} discuss and use the connection between gradient-penalties
and adversarial augmentation, but surprisingly never empirically compare both,
which we do in Section~\ref{penResults}. This experiment is crucial to confirm
the validity of the first-order approximation made in \eqref{dualnorm} to link
adversarial damage and loss-gradients. \Citet{hein17formal} derived yet another
gradient-based penalty --the \emph{cross-Lipschitz}-penalty-- by considering
and proving formal guarantees on adversarial vulnerability (see
App.\ref{crosslip}).  Penalizing network-gradients is also at the heart of
contractive auto-encoders as proposed by \citet{rifai11contractive}, where it
is used to regularize the encoder-features.  A gradient regularization of the
loss of generative models also appears in Proposition~6 of
\citet{ollivier14autoencoders}, where it stems from a code-length bound on the
data (minimum description length). For further references on adversarial
attacks and defenses, see e.g.~\citet{yuan17adversarial}.

\section{From Adversarial Examples to Large Gradients\label{sec:av_and_gradients}}

Suppose that a given classifier
$\varphi$ classifies an image $\xx$ as being in category
$\varphi(\xx)$. An adversarial image is a small modification of
$\xx$, barely noticeable to the human eye, that suffices
to fool the classifier into predicting a class different from $\varphi(\xx)$.
It is a \emph{small} perturbation of the inputs, that creates a
\emph{large} variation of outputs. Adversarial examples thus seem inherently related
to large gradients of the network. A connection that we will now clarify.
Note that visible adversarial examples
sometimes appear in the literature, but we deliberately focus on imperceptible
ones.

\paragraph{Adversarial vulnerability and adversarial damage.}
In practice, an adversarial image is constructed by adding a perturbation
$\dd$ to the original image $\xx$ such that $\norm{\dd} \leq \epsilon$ for some (small) number
$\epsilon$ and a given norm $\norm{\cdot}$ over the input space. 
We call the perturbed input $\xx + \dd$ an $\epsilon$-sized
$\norm{\cdot}$-attack and say that the attack was successful when $\varphi(\xx
+ \dd) \neq \varphi(\xx)$. This motivates
\begin{definition}\label{advvul}
    Given a distribution $P$ over the input-space, we call \emph{adversarial
    vulnerability} of a classifier $\varphi$ to an $\epsilon$-sized
    $\norm{\cdot}$-attack the probability that there exists a perturbation
    $\dd$ of $\xx$ such that 
    \begin{equation}\label{ae_def}
        \norm{\dd} \leq \epsilon \quad \text{and} \quad
        \varphi(\xx) \neq \varphi(\xx+\dd) \,.
    \end{equation}

    We call the average increase-after-attack $\ee{\xx \sim P}{\Delta \loss}$
    of a loss $\loss$ the \emph{adversarial ($\loss$-) damage} (of the
    classifier $\varphi$ to an $\epsilon$-sized $\norm{\cdot}$-attack).
\end{definition}
When $\loss$ is the 0-1-loss $\loss_{0/1}$, adversarial damage is the
accuracy-drop after attack.
The 0-1-loss damage is always smaller than
adversarial vulnerability, because vulnerability counts all class-changes of
$\varphi(\xx)$, whereas some of them may be neutral to adversarial damage
(e.g.\ a change between two wrong classes). The
$\loss_{0/1}$-adversarial
damage thus lower bounds adversarial vulnerability. Both are even equal when
the classifier is perfect (before attack), because then every change of label
introduces an error.
It is hence tempting to evaluate adversarial vulnerability
with $\loss_{0/1}$-adversarial damage.

\paragraph{From $\Delta \loss_{0/1}$ to $\Delta \loss$ and to $\dl$.}
In practice however, we do not train our classifiers with the
non-differentiable 0-1-loss but use a smoother surrogate loss $\loss$, such as
the cross-entropy loss. For similar reasons, we will now investigate the
adversarial damage $\ee{\xx}{\Delta \loss(\xx, c)}$ with loss $\loss$ rather
than $\loss_{0/1}$.  Like for
\citet{goodfellow15explaining,lyu15unified,sinha18certifiable} and many others,
a classifier $\varphi$ will hence be robust if, on average over $\xx$, a small
adversarial perturbation $\dd$ of $\xx$ creates only a small variation $\delta
\loss$ of the loss. Now, if $\norm{\dd} \leq \epsilon$, then a first order
Taylor expansion in $\epsilon$ shows that
\begin{equation}\label{dualnorm}
    \begin{aligned}
    \delta \loss
        &= \max_{\dd \,:\, \norm{\dd} \leq \epsilon}
               | \loss(\xx+\dd, c) - \loss(\xx, c) |\\
        &\approx \max_{\dd \,:\, \norm{\dd} \leq \epsilon}
                \left | \dl
                    \cdot \dd \right |
        = \epsilon \, \dn{\dl} ,
    \end{aligned}
\end{equation}
where $\dl$ denotes the gradient of $\loss$ with respect to $\xx$, and where
the last equality stems from the definition of the dual norm $\dn{\cdot}$
of $\norm{\cdot}$.
Now two remarks. First: the dual norm only kicks in
because we let the input noise $\dd$ optimally adjust to the coordinates of
$\dl$ within its $\epsilon$-constraint. This is the brand mark of
\emph{adversarial} noise: the different coordinates add up,
instead of statistically canceling each other out as they would with random
noise. For example, if we impose
that $\norm{\dd}_2 \leq \epsilon$, then $\dd$ will strictly align with $\dl$.
If instead $\norm{\dd}_\infty \leq \epsilon$, then $\dd$ will align with the
sign of the coordinates of $\dl$. Second remark: while the Taylor expansion
in~\eqref{dualnorm} becomes exact for infinitesimal perturbations, for finite
ones it may actually be dominated by higher-order terms.
Our experiments (Figures~\ref{fig:normDependence} \&~\ref{fig:results}) however strongly suggest
that in practice the first order term dominates the others.
Now, remembering that the dual norm of an $\lnorm_p$-norm is the corresponding
$\lnorm_q$-norm, and summarizing,  we have proven
\begin{lemma}\label{advExAndDualNorm}
    At first order approximation in $\epsilon$, an $\epsilon$-sized adversarial
    attack generated with norm $\norm{\cdot}$ increases the loss $\loss$ at
    point $\xx$ by $\epsilon \, \dn{\dl}$, where $\dn{\cdot}$ is the dual norm of
    $\norm{\cdot}$. 
    In particular, an $\epsilon$-sized $\lnorm_p$-attack increases the loss
    by~$\epsilon \norm{\dl}_q$ where $1 \leq p \leq \infty$ and 
    $\frac 1 p + \frac 1 q = 1$.
\end{lemma}%
Although the lemma is valid at first order only, it proves that \emph{at
least} this kind of first-order vulnerability is present. Moreover, we will see
that the first-order predictions closely match the experiments, and that simple
gradient regularization helps protecting even against iterative
(non-first-order) attack methods (Figure~\ref{fig:normDependence}).

\paragraph{Calibrating the threshold $\epsilon$ to the attack-norm $\norm{\cdot}$.}
Lemma~\ref{advExAndDualNorm} shows that adversarial vulnerability depends on
three main factors:
\begin{enumerate*}[label=(\roman*), nosep, itemindent=4ex]
    \item \label{it1} $\norm{\cdot}$\,, the norm chosen for the attack
    \item \label{it2} $\epsilon$\,, the size of the attack, and
    \item \label{it3} $\EE{\xx}{\dn{\dl}}$\,, the expected \emph{dual} norm of $\dl$. 
\end{enumerate*}
We could see Point~\ref{it1} as a measure of our sensibility to
image perturbations, \ref{it2} as our sensibility
threshold, and \ref{it3} as the classifier's expected marginal sensibility to a unit 
perturbation. $\EE{\xx}{\dn{\dl}}$ hence intuitively captures the discrepancy between our
perception (as modeled by $\norm{\cdot}$) and
the classifier's perception for an input-perturbation of small size $\epsilon$.
Of course, this viewpoint supposes that we actually found a norm $\norm{\cdot}$
(or more generally a metric) that faithfully reflects human perception --~a
project in its own right, far beyond the scope of this paper.
However, it is clear that the threshold $\epsilon$ that we choose should depend
on the norm $\norm{\cdot}$ and hence on the input-dimension $d$. 
In particular, for a given pixel-wise order of magnitude
of the perturbations $\dd$, the $\lnorm_p$-norm of the perturbation will scale like
$d^{1/p}$. This suggests to write the
threshold $\epsilon_p$ used with $\ell_p$-attacks as:
\begin{equation}\label{rescale}
    \epsilon_p = \epsilon_\infty \, d^{1/p} \, ,
\end{equation}
where $\epsilon_\infty$ denotes a dimension independent constant. In
Appendix~\ref{perception} we show that this scaling also preserves the average
signal-to-noise ratio $\norm{\xx}_2 / \norm{\dd}_2$, both across norms and
dimensions, so that $\epsilon_p$ could correspond to a constant human perception-threshold.
With this in mind, the impatient reader may already jump to
Section~\ref{sec:grad_estimation}, which contains our main contributions:
the estimation of $\EE{x}{\norm{\dl}_q}$ for standard feedforward nets.
Meanwhile, the rest of this section shortly discusses two straightforward
defenses that we will use later and that further illustrate the role of
gradients.

\paragraph{A new old regularizer.} Lemma~\ref{advExAndDualNorm} shows
that the loss of the network after an
$\frac \epsilon 2$-sized $\norm{\cdot}$-attack is 
\begin{equation}\label{newloss}
    \loss_{\epsilon, \dn{\cdot}}(\xx,c)
        := \loss(\xx,c) + \frac{\epsilon}{2} \,  \dn{\dl} \, .
\end{equation}
It is thus natural to take this loss-after-attack as a new training objective.
Here we introduced a factor $2$ for reasons that will become clear in a
moment. Incidentally, for $\norm{\cdot} = \norm{\cdot}_2$, this new loss
reduces to an old regularization-scheme proposed by \citet{drucker91double}
called \emph{double-backpropagation}.
At the time, the authors argued that slightly decreasing a function's or a classifier's
sensitivity to input perturbations should improve generalization.
In a sense, this is exactly our motivation when defending against adversarial examples.
It is thus not surprising to end up with the same regularization term.  
Note that our reasoning only shows that training with one specific norm
$\dn{\cdot}$ in \eqref{newloss} helps to protect against adversarial
examples generated from $\norm{\cdot}$. A priori, we do not know what will
happen for attacks generated with other norms; but our experiments suggest that training
with one norm also protects against other attacks (see
Figure~\ref{fig:results} and Section~\ref{penResults}).

\paragraph{Link to adversarially augmented training.}
In \eqref{ae_def}, $\epsilon$ designates an attack-size threshold, while in
\eqref{newloss}, it is a regularization-strength. Rather than a
notation conflict, this reflects an intrinsic duality between two
complementary interpretations of $\epsilon$, which we now investigate further.
Suppose that, instead of using the loss-after-attack,
we augment our training set with $\epsilon$-sized $\norm{\cdot}$-attacks
$\xx+\dd$, where for each training point $\xx$, the perturbation $\dd$ is
generated on the fly to locally maximize the loss-increase.
Then we are effectively training with
\begin{equation}\label{fgsmLoss}
    \tilde{\loss}_{\epsilon, \norm{\cdot}}(\xx,c) 
        := \frac 1 2 (\loss(\xx,c) + \loss(\xx+\epsilon \,  \dd, c)) \, ,
\end{equation}
where by construction $\dd$ satisfies \eqref{dualnorm}.
We will refer to this technique as \emph{adversarially augmented training}.
It was first introduced by \citet{goodfellow15explaining} with
$\norm{\cdot} = \norm{\cdot}_\infty$ under the name of
FGSM\footnote{\emph{F}ast \emph{G}radient \emph{S}ign \emph{M}ethod}-augmented
training.
Using the first order Taylor expansion in $\epsilon$ of
\eqref{dualnorm}, this `old-plus-post-attack' loss of
\eqref{fgsmLoss} simply reduces to our loss-after-attack, which proves
\begin{proposition}\label{lossEqui}
    Up to first-order approximations in $\epsilon$, 
    $ 
        \tilde{\loss}_{\epsilon, \norm{\cdot}} = \loss_{\epsilon, \dn{\cdot}} \, .
    $ 
    Said differently, for small enough $\epsilon$,
    adversarially augmented training with $\epsilon$-sized $\norm{\cdot}$-attacks
    amounts to penalizing the \emph{dual} norm $\dn{\cdot}$ of $\dl$ with weight
    $\epsilon/2$.
    In particular, double-backpropagation corresponds to training with
    $\ell_2$-attacks, while FGSM-augmented training corresponds to an
    $\ell_1$-penalty on $\dl$.
\end{proposition}
This correspondence between training with perturbations and using a regularizer
can be compared to Tikhonov regularization: Tikhonov regularization amounts to
training with \emph{random} noise \citet{bishop95training}, while training with
\emph{adversarial} noise amounts to penalizing $\dl$. 
Section~\ref{penResults}
verifies the correspondence between adversarial augmentation and gradient
regularization empirically, which also strongly suggests the empirical validity
of the first-order Taylor expansion in~\eqref{dualnorm}.

\section{Estimating \texorpdfstring{$\norm{\dl}_q$}{|d\_x d\,L|\_q} to Evaluate Adversarial Vulnerability\label{sec:grad_estimation}}

In this section, we evaluate the size of $\norm{\dl}_q$ for a very wide class
of standard network architectures. We show that, inside this class, the
gradient-norms are independent of the network topology and increase with input
dimension. We start with an intuitive explanation of these insights
(Sec~\ref{sec:coreidea}) before moving to our formal statements
(Sec~\ref{sec:theorems}).

\subsection{Core Idea: One Neuron with Many Inputs\label{sec:coreidea}}

This section is for intuition only: no assumption made here is used later.
We start by showing how changing $q$ affects the size of
$\norm{\dl}_q$. Suppose for a moment that the coordinates of $\dl$ have
typical magnitude $|\dl|$. Then $\norm{\dl}_q$ scales like $d^{1/q} |\dl|$.
Consequently
\begin{equation}\label{res_dependance}
    \epsilon_p \norm{\dl}_q \
        \propto \ \epsilon_p \, d^{1/q} \, |\dl| \ 
        \propto \ d \, |\dl| \, .
\end{equation}
This equation carries two important messages. First, we see how $\norm{\dl}_q$
depends on $d$ and $q$. The dependence seems highest for $q=1$. But once we
account for the varying perceptibility threshold $\epsilon_p \propto
d^{1/p}$, we see that adversarial vulnerability scales like $d \cdot |\dl|$,
whatever $\lnorm_p$-norm we use.
Second, \eqref{res_dependance} shows that to be robust against any type
of $\lnorm_p$-attack at any input-dimension $d$, the average absolute value of
the coefficients of $\dl$ must grow slower than $1/d$. Now, here is the
catch, which brings us to our core insight.

In order to preserve the activation variance of the neurons from layer
to layer, the neural weights are usually initialized with a variance that is
inversely proportional to the number of inputs per neuron. 
Imagine for a moment that the network consisted only of one output neuron $o$
linearly connected to all input pixels. For the purpose of this example,
we assimilate $o$ and $\loss$. Because we initialize the weights with
a variance of $1/d$, their average absolute value $|\bm \partial_x o| \equiv |\dl|$
grows like $1/\sqrt d$, rather than the required $1/d$. 
By \eqref{res_dependance}, the adversarial vulnerability $\epsilon
\norm{\bm \partial_x o}_q \equiv \epsilon \norm{\dl}_q$ therefore increases
like ${d / \sqrt d = \sqrt d}$. 

\emph{This toy example shows that the standard initialization scheme, which preserves
the variance from layer to layer, causes the average coordinate-size $|\dl|$ to
grow like $1/\sqrt d$ instead of $1/d$. When an $\lnorm_\infty$-attack tweaks its
$\epsilon$-sized input-perturbations to align with the coordinate-signs of
$\dl$, all coordinates of $\dl$ add up in absolute value, resulting in an
output-perturbation that scales like $\epsilon \sqrt d$ and leaves the network
increasingly vulnerable with growing input-dimension.}

\subsection{Formal Statements for Deep Networks\label{sec:theorems}}

Our next theorems formalize and generalize the previous toy example to a very
wide class of feedforward nets with ReLU activation functions.  For
illustration purposes, we start with fully connected nets before proceeding
with the broader class, which includes any succession of (possibly strided)
convolutional layers.  In essence, the proofs iterate our insight on one layer
over a sequence of layers.  They all rely on the following set
\begin{enumerate*}[label=($\mathcal H$)]
    \item \label{H}\!
\end{enumerate*}
of hypotheses:
\begin{enumerate}[label=H\arabic*, nosep]
    \item \label{h1} Non-input neurons are followed by a ReLU
    killing half of its inputs, independently of the weights.
    \item \label{h5} Neurons are partitioned into layers, meaning 
    groups that each path traverses at most once.
    \item \label{h2} All weights have $0$ expectation and variance
    $2/(\text{in-degree})$ (`He-initialization').
    \item \label{h3} The weights from different layers are independent.
    \item \label{h4} Two distinct weights $w, w'$ from a same node satisfy
    $\e{w \, w'} = 0$.
\end{enumerate}
If we follow common practice and initialize our nets as proposed by
\citet{he15delving}, then~\ref{h2}-\ref{h4} are satisfied at
initialization by design, while \ref{h1} is usually a very good
approximation \citep{balduzzi17neural}. 
Note that such i.i.d.\ weight assumptions have been widely used to analyze
neural nets and are at the heart of very influential and
successful prior work (e.g., equivalence between neural nets and Gaussian processes as
pioneered by \citealt{neal96thesis}).
Nevertheless, they do not hold after training.
That is why all our statements in this
section are to be understood as \emph{orders of magnitudes} that are very well
satisfied at initialization both in theory and practice, and that we will confirm
experimentally for trained networks in Section~\ref{experiments}. Said differently,
while our theorems rely on the statistics of neural nets at initialization, our
experiments confirm their conclusions after training.

\begin{theorem}[\textbf{Vulnerability of Fully Connected Nets}]\label{fullcon}
    Consider a succession of fully connected layers with ReLU activations
    which takes inputs $\xx$ of dimension $d$, 
    satisfies assumptions \H, and outputs logits $f_k(\xx)$ that get fed to a final
    cross-entropy-loss layer $\loss$. Then the coordinates of $\dfk$ grow like
    $1/\sqrt d$, and
    \begin{equation}\label{dl_norm}
        \norm{\dl}_q \propto d^{\frac{1}{q} - \frac{1}{2}}
        \quad \mathrm{and} \quad
        \epsilon_p \norm{\dl}_q \propto \sqrt d \, .
    \end{equation}
    These networks are thus increasingly vulnerable to
    $\lnorm_p$-attacks with growing input-dimension.
\end{theorem}

Theorem~\ref{fullcon} is a special case of the next theorem, which will show
that the previous conclusions are essentially independent of the
network-topology.
We will use the following
symmetry assumption on the neural connections. For a given path $\pp$, let the
\emph{path-degree} $d_\pp$ be the multiset of encountered in-degrees
along path $\pp$. For a fully connected network, this is the unordered sequence
of layer-sizes preceding the last path-node, including the input-layer.
Now consider the multiset $\{d_\pp\}_{\pp \in \P(x,o)}$ of all path-degrees
when $\pp$ varies among all paths from input $x$ to output $o$.
The symmetry assumption (relatively to $o$) is
\begin{enumerate}[label=($\mathcal S$), nosep]
    \item \label{S} All input nodes $x$ have the same multiset $\{d_\pp\}_{\pp \in
    \P(x,o)}$ of path-degrees from $x$ to $o$.
\end{enumerate}
Intuitively, this means that the statistics of degrees encountered along paths to
the output are the same for all input nodes.
This symmetry assumption is exactly satisfied by fully
connected nets, almost satisfied by CNNs (up to boundary effects, which can be
alleviated via periodic or mirror padding) and exactly satisfied by strided layers,
if the layer-size is a multiple of the stride.


\begin{theorem}[\textbf{Vulnerability of Feedforward Nets}]\label{generaltheo}
    Consider any feedforward network with linear connections and ReLU
    activation functions.
    Assume the net satisfies assumptions \H\ and outputs logits
    $f_k(\xx)$ that get fed to the  cross-entropy-loss $\loss$. Then
    $\norm{\bm{\partial_{\xx}} f_k}_2$ is independent of the input dimension
    $d$ and $\epsilon_2 \norm{\dl}_2 \propto \sqrt d$.
    Moreover, if the net satisfies the symmetry assumption \S, then 
    $|\dfik| \propto 1/\sqrt d$
    and \eqref{dl_norm} still holds: $\norm{\dl}_q \propto
    d^{\frac{1}{q} - \frac{1}{2}}$ and $\epsilon_p \norm{\dl}_q \propto \sqrt
    d$.
\end{theorem}
Theorems~\ref{fullcon} and~\ref{generaltheo} are proven in
Appendix~\ref{proofs}. The main proof idea is that in the gradient norm computation,
the He-initialization exactly compensates the combinatorics of the number of
paths in the network, so that this norm becomes independent of the network topology.
In particular, we get
\begin{corollary}[\textbf{Vulnerability of CNNs}]\label{convnet}
    In any succession of convolution and dense layers, strided or not,
    with ReLU activations, that satisfies
    assumptions \H\ and outputs logits that get fed to the cross-entropy-loss
    $\loss$, the gradient of the logit-coordinates scale like $1/\sqrt d$ and
    \eqref{dl_norm} is satisfied. It is hence
    increasingly vulnerable with growing input-resolution to attacks generated
    with any $\lnorm_p$-norm.
\end{corollary}

\paragraph{Remarks.} 
\begin{itemize}[nosep,wide] 
    \item Appendix~\ref{poolsec} shows that the network gradients are dampened
    when replacing strided layers by average poolings, essentially because
    average-pooling weights do not follow the He-init assumption \ref{h2}.
    \item Although the principles of our analysis naturally extend to residual
    nets, they are not yet covered by our theorems (residual connections do not
    satisfy \ref{h2}).
    \item Current weight initializations (He-, Glorot-, Xavier-) are chosen to
    preserve the variance from layer to layer, which constrains their scaling
    to $\nicefrac{1}{\sqrt{\text{in-degree}}}$. This scaling, we show, is
    incompatible with small gradients. But decreasing gradients simply by
    reducing the initial weights would kill the output signal and make training
    impossible for deep nets \citep[Sec 2.2]{he15delving}. Also note that
    rescaling all weights by a constant does not change the classification
    decisions, but it affects cross-entropy and therefore adversarial damage.
\end{itemize}

\section{Empirical Results\label{experiments}}

\begin{figure*}[htb]
    \centering
    \includegraphics[width=\linewidth]{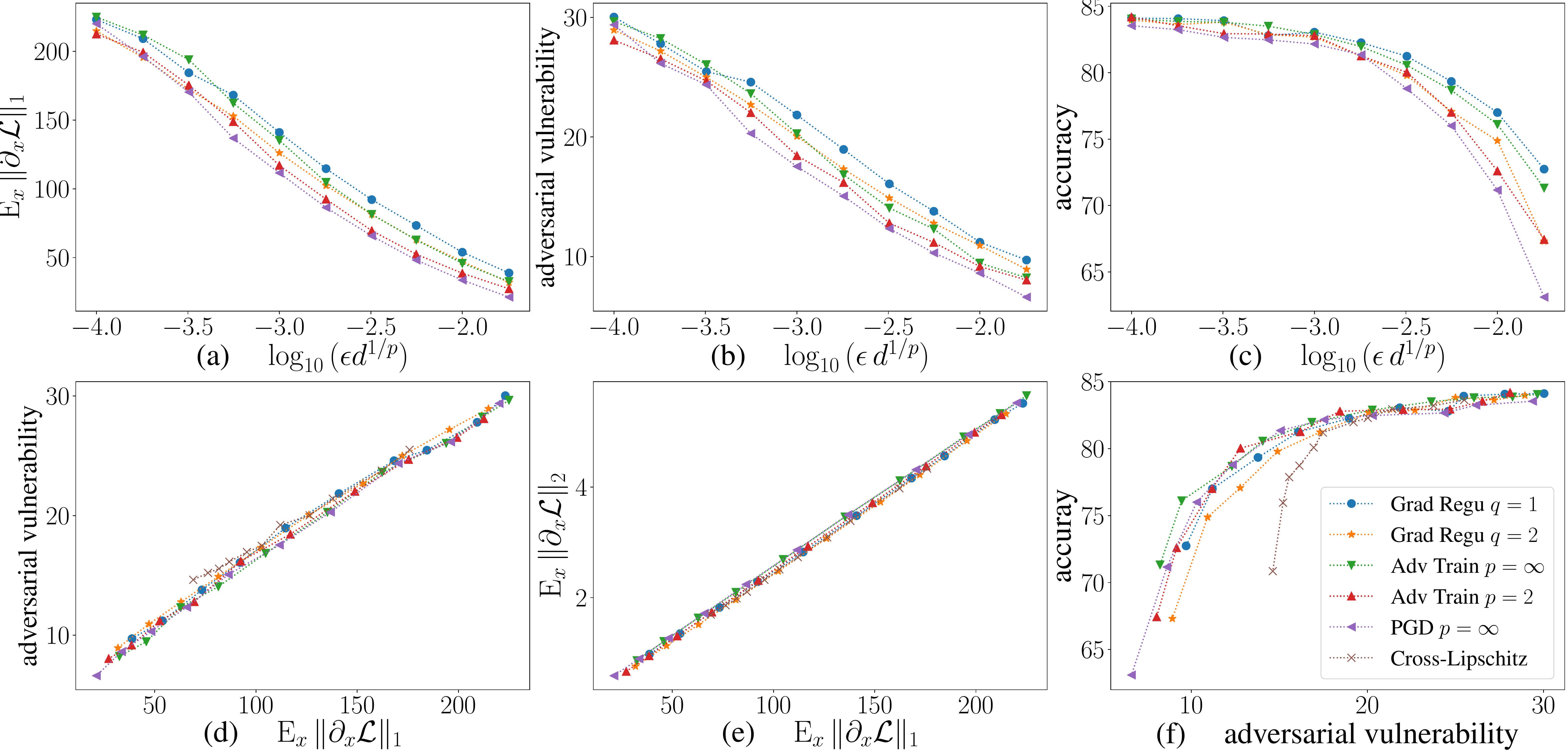}
    \vspace{-2ex}
    \caption{Average norm $\EE{\xx}{\norm{\dl}}$ of the loss-gradients,
    adversarial vulnerability and accuracy (before attack) of various networks trained
    with different adversarial regularization methods and regularization
    strengths $\epsilon$. Each point represents a trained network, and each
    curve a training-method.
    \emph{Upper row}: A priori, the regularization-strengths $\epsilon$ have
    different meanings for each method. The near superposition of all upper-row curves 
    illustrates $(i)$ the duality between
    adversarial augmentation and gradient regularization (Prop.\ref{lossEqui}) and $(ii)$
    confirms the rescaling of $\epsilon$ proposed in \eqref{rescale} and
    $(iii)$ supports the validity of the first-order Taylor
    expansion~\eqref{dualnorm}.
    $(d)$: near functional relation between adversarial vulnerability and
    average loss-gradient norms.
    $(e)$: the near-perfect linear relation between the $\E{\norm{\dl}_1}$ and
    $\E{\norm{\dl}_2}$ suggests that protecting against a given
    attack-norm also protects against others.
    $(f)$: Merging \fb and \fc\ shows that
    all adversarial augmentation and gradient regularization
    methods achieve similar accuracy-vulnerability
    trade-offs.\label{fig:results}
    }
    \vspace{-2ex}
\end{figure*}

Section~\ref{penResults} empirically verifies the validity of the
first-order Taylor approximation made in~\eqref{dualnorm} and the
correspondence between gradient regularization and adversarial augmentation
(Fig.\ref{fig:results}). Section~\ref{resoDependance} analyzes
the dimension dependence of the average gradient-norms and adversarial
vulnerability after usual and robust training. Section~\ref{penResults} uses an
attack-threshold $\epsilon_\infty = 0.5\%$ of the pixel-range (invisible to
humans), with PGD-attacks from the Foolbox-package \citep{rauber17foolbox}.
Section~\ref{resoDependance} uses self-coded PGD-attacks with random start with
$\epsilon_\infty = 0.08\%$. As a safety-check, other attacks were tested as
well (see App.\ref{penResults} \& Fig.\ref{fig:normDependence}), but results remained essentially unchanged. Note that the
$\epsilon_\infty$-\emph{thresholds} should not be confused with the
\emph{regularization-strengths} $\epsilon$ appearing in~\eqref{newloss}
and~\eqref{fgsmLoss}, which will be varied. The datasets were
normalized ($\sigma\approx .2$). All regularization-values $\epsilon$
are reported in these normalized units (i.e.\ multiply by $.2$ to
compare with 0-1 pixel values). Code available at {\small
\url{https://github.com/facebookresearch/AdversarialAndDimensionality}}.

\subsection{First-Order Approximation, Gradient Penalty, Adversarial Augmentation\label{penResults}}

We train several CNNs with same architecture to classify CIFAR-10 images
\citep{krizhevsky09learning}. For each net, we use a specific training method
with a specific regularization value $\epsilon$. The training methods used were
$\ell_1$- and $\ell_2$-penalization of $\dl$ (Eq.~\ref{newloss}), adversarial
augmentation with $\ell_\infty$- and $\ell_2$- attacks (Eq.~\ref{fgsmLoss}),
projected gradient descent (PGD) with randomized starts (7 steps per attack
with step-size $= .2 \, \epsilon_\infty$; see \citealt{madry18towards}) and the
cross-Lipschitz regularizer (Eq.~\ref{heinregu} in Appendix~\ref{crosslip}).
For this experiment, all networks have 6 `strided convolution $\rightarrow$
batchnorm $\rightarrow$ ReLU' layers with strides [1, 2, 2, 2, 2, 2]
respectively and 64 output-channels each, followed by a final fully-connected
linear layer.  Results are summarized in Figure~\ref{fig:results}. Each curve
represents one training method. Note that our goal here is not to advocate one
defense over another, but rather to check the validity of the Taylor expansion,
and empirically verify that first order terms (i.e., gradients) suffice to
explain much of the observed adversarial vulnerability.

\textbf{Confirming first order expansion and large first-order vulnerability.}
The following observations support the validity of the first order Taylor
expansion in~\eqref{dualnorm} and suggest that it is a crucial component of
adversarial vulnerability:
\begin{enumerate*}[label=(\roman*)]
    \item the efficiency of the first-order defense against iterative
    (non-first-order) attacks (Fig.\ref{fig:results}\&\ref{fig:normDependence}a);
    \item the striking similarity between the PGD curves (adversarial
    augmentation with \emph{iterative} attacks) and the other adversarial
    training training curves (\emph{one-step} attacks/defenses);
    \item the functional-like dependence between any approximation of
    adversarial vulnerability and $\EE{\xx}{\norm{\dl}_1}$
    (Fig.\ref{fig:normDependence}b), and its independence on the training
    method (Fig.\fd).
    \item the excellent correspondence between the gradient regularization and
    adversarial augmentation curves (see next paragraph).
\end{enumerate*}
Said differently, adversarial examples seem indeed to be primarily caused by
large gradients of the classifier as captured via the induced loss.

\textbf{Gradient regularization matches adversarial augmentation  (Prop.\ref{lossEqui}).}
The upper row of Figure~\ref{fig:results} plots $\EE{\xx}{\norm{\dl_1}}$,
adversarial vulnerability and accuracy as a function of $\epsilon \, d^{1/p}$.
The excellent match between the adversarial augmentation curve with $p=\infty$
($p=2$) and its gradient regularization dual counterpart with $q=1$ (resp.\
$q=2$) illustrates the duality between $\epsilon$ as a threshold for
adversarially augmented training and as a regularization constant in the
regularized loss (Proposition~\ref{lossEqui}).  It also supports the validity
of the first-order Taylor expansion in \eqref{dualnorm}.

\begin{figure*}[tb]
    \centering
    \includegraphics[width=\linewidth]{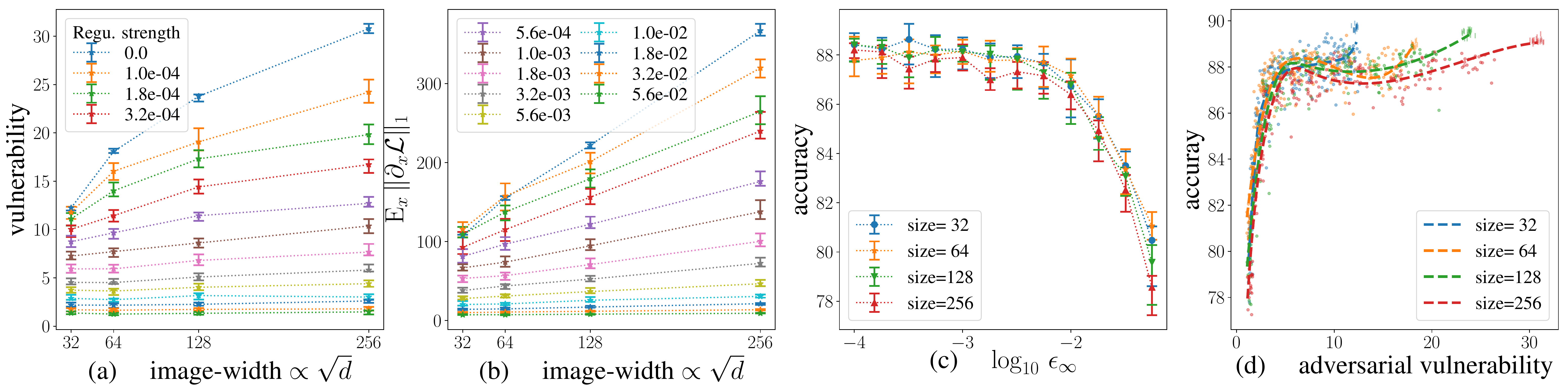}
    \caption{\label{fig:bcifar_summary}Input-dimension dependence of
    adversarial vulnerability, gradient norms and accuracy measured on
    up-sampled CIFAR-10 images. (b) Similar to our theorems' prediction at
    initialization, average gradient norms increase like $\sqrt d$ yielding (a)
    higher vulnerability. Larger PGD-regularization during training can
    significantly dampen this dimension dependence with (c) almost no harm to
    accuracy at first (long plateau on \ffc). Accuracy starts getting damaged
    when the dimension dependence is nearly broken ($\epsilon_\infty \approx
    .0005$). (d) Whatever the input-dimension, PGD-training achieves similar
    accuracy-vulnerability trade-offs. (c) \& (d) suggest that PGD-training
    effectively recovers the original image size, 3x32x32.}
    \vspace{-2ex}
\end{figure*}

\textbf{Confirming correspondence of norm-dependent thresholds (Eq.\ref{rescale}).}
Still on the upper row, the curves for $p=\infty, q=1$ have no reason to match
those for $p=q=2$ when plotted against $\epsilon$, because the
$\epsilon$-threshold is relative to a specific attack-norm. However,
\eqref{rescale} suggested that the rescaled thresholds $\epsilon d^{1/p}$ may
approximately correspond to a same `threshold-unit' across $\ell_p$-norms and
across dimension. This is well confirmed by the upper row plots: by rescaling
the x-axis, the $p=q=2$ and $q=1$, $p=\infty$ curves get almost super-imposed.

\textbf{Accuracy-vulnerability trade-off: confirming large first-order
component of vulnerability.}
Merging Figures~\fb\ and~\fc\ by taking out $\epsilon$, Figure~\ff\ shows that
all gradient regularization and adversarial augmentation methods, \emph{including
iterative ones (PGD)}, yield equivalent accuracy-vulnerability trade-offs. This
suggest that adversarial vulnerability is largely first-order. For higher
penalization values, these trade-offs appear to be much better than those given
by cross Lipschitz regularization.

\textbf{The regularization-norm does not matter.}
We were surprised to see that on Figures~\fd\ and~\ff, the $\loss_{\epsilon,
q}$ curves are almost identical for $q=1$ and $2$. This indicates that both
norms can be used interchangeably in~\eqref{newloss} (modulo proper rescaling
of $\epsilon$ via \eqref{rescale}), and suggests that protecting against a
specific attack-norm also protects against others.  \eqref{res_dependance} may
provide an explanation: if the coordinates of $\dl$ behave like centered,
uncorrelated variables with equal variance --which would follow from
assumptions~\H~--, then the $\lnorm_1$- and $\lnorm_2$-norms of $\dl$ are
simply proportional. Plotting $\EE{x}{\norm{\dl(\xx)}_2}$ against
$\EE{x}{\norm{\dl(\xx)}_1}$ in Figure~\fe\ confirms this explanation. The slope
is independent of the training method. (But Fig~\fge shows that it is not
independent of the input-dimension.) Therefore, penalizing $\norm{\dl(\xx)}_1$
during training will not only decrease $\EE{\xx}{\norm{\dl}_1}$ (as shown in
Figure~\fa), but also drive down $\EE{\xx}{\norm{\dl}_2}$ and vice-versa.

\subsection{Vulnerability's Dependence on Input Dimension\label{resoDependance}}

Theorems~\ref{fullcon}-\ref{generaltheo} and Corollary~\ref{convnet} predict a
linear growth of the average $\lnorm_1$-norm of $\dl$ with the square root of
the input dimension $d$, and therefore an increased adversarial vulnerability
(Lemma~\ref{advExAndDualNorm}).  To test these predictions, we compare the
vulnerability of different PGD-regularized networks when varying the
input-dimension.  To do so, we resize the original 3x32x32 CIFAR-10 images to
32, 64, 128 and 256 pixels per edge by copying adjacent pixels, and train one CNN
for each input-size and regularization strength $\epsilon$.  All nets had the
same amount of parameters and very similar structure across input-resolutions
(see Appendix~\ref{sec:architecture}). All reported values were computed over
the last 20 training epochs on the same held-out test-set.

\textbf{Gradients and vulnerability increase with $\sqrt d$.}
Figures~\ffa\&\ffb summarize the resulting dimension dependence of
gradient-norms and adversarial vulnerability. The dashed-lines follow the
medians of the 20 last epochs and the errorbars show their \nth{10} and
\nth{90} quantiles. Similar to the predictions of our theorems at
initialization, we see that, even after training, $\ee{\xx}{\norm{\dl}_1}$ grows
linearly with $\sqrt d$ which yields higher adversarial vulnerability.
However, increasing the regularization decreases the slope of
this dimension dependence until, eventually, the dependence breaks.

\textbf{Accuracies are dimension independent.}
Figure~\ffc plots accuracy versus regularization strength, with errorbars
summarizing the 20 last training epochs.\footnote{%
Fig.\ffc\&\ffc are similar to Figures~\fc\&\ff, but with one curve per
input-dimension instead of one per regularization method. See
Appendix~\ref{sec:additional_figures} for full equivalent of
Figure~\ref{fig:results}.}
The four curves correspond to the four different input dimensions. They overlap,
which confirms that contrary to vulnerability, the accuracies are dimension
independent; and that the $\lnorm_\infty$-attack thresholds are essentially
dimension independent.

\textbf{PGD effectively recovers original input dimension.}
Figure~\ffd plots the accuracy-vulnerability trade-offs achieved by the
previous nets over their 20 last training epochs, with a smoothing spline
fitted for each input dimension (scipy's UnivariateSpline with s=200). Higher
dimensions have a longer plateau to the right, because without regularization,
vulnerability increases with input dimension. The curves overlap when moving to
the left, meaning that the accuracy-vulnerability trade-offs achieved by PGD
are essentially \emph{independent of the actual input dimension.}

\textbf{PGD training outperforms down-sampling.}
On artificially upsampled CIFAR-10 images, PGD regularization acts as if it
first reduced the images back to their original size before classifying them.
Can PGD outperform this strategy when the original image is really high
resolution? To test this, we create a 12-class `Mini-ImageNet' dataset with
approximately 80,000 images of size 3x256x256 by merging similar ImageNet
classes and center-cropping/resizing as needed. We then do the same experiment
as with up-sampled CIFAR-10, but using down-sampling instead of up-sampling
(Appendix~\ref{sec:imgnet}, Fig.~\ref{fig:imgnet_summary}). While the
dependence of vulnerability to the dimension stays essentially unchanged, PGD
training now achieves much better accuracy-vulnerability trade-offs with the
original high-dimensional images than with their down-sampled versions.

\textbf{Insights from figures in Appendix~\ref{sec:additional_figures}.}
Appendix~\ref{sec:additional_figures} reproduces many additional figures on
this section's experiments. They yield additional insights which we summarize
here.

    \emph{Non-equivalence of loss- and accuracy-damage.}
    Figure~\ref{fig:bcifar_te_training_curves}a\&c show that the test-error
    continues to decrease all over training, while the cross-entropy increases
    on the test set from epoch $\approx 40$ and on.  This aligns with the
    observations and explanations of \citet{soudry18implicit}. But it also
    shows that one must be careful when substituting their differentials, loss-
    and accuracy-damage. (See also Fig.\ref{fig:bcifar_early_vs_last}b.)

    \emph{Early stopping dampens vulnerability.}
    Fig.\ref{fig:bcifar_te_training_curves} shows that adversarial damage and
    vulnerability closely follow the evolution of cross-entropy. Since
    cross-entropy overfits, early stopping effectively acts as a defense.  See
    Fig.\ref{fig:bcifar_early_vs_last2}.

    \emph{Gradient norms do not generalize well.}
    Figure~\ref{fig:tr_te_norms} reveals a strong discrepancy between the
    average gradient norms on the test and the training data. This discrepancy
    increases over training (gradient norms decrease on the training data but
    increase on the test set), and with the input dimension, as $\sqrt d$. This
    dimension dependence might suggest that, outside the training points, the
    networks tend to recover initial gradient properties.
    Our observations confirm
    \citeauthor{schmidt18adversarially}'s \citeyearpar{schmidt18adversarially}
    recent finding that PGD-regularization has a hard time generalizing to the
    test-set. They claim that better generalization requires more data.
    Alternatively, we could try to rethink our network modules to adapt it to
    the data, e.g.\ by decreasing their degrees of `gradient-freedom'.
    Evaluating the gradient-sizes at initialization may help doing so.

\section{Conclusion}
For differentiable classifiers and losses, we showed that adversarial
vulnerability increases with the gradients $\dl$ of the loss. All
approximations made are fully specified, and validated by the near-perfect
functional relationship between gradient norms and vulnerability (Fig.\fd). We
evaluated the size of $\norm{\dl}_q$ and showed that, at initialization, many
current feedforward nets (convolutional or fully connected) are increasingly
vulnerable to $\lnorm_p$-attacks with growing input dimension (image size),
independently of their architecture. Our experiments confirm this dimension
dependence after usual training, but PGD-regularization dampens it and can
effectively counter-balance the effect of artificial input dimension
augmentation. Nevertheless, regularizing beyond a certain point yields a rapid
decrease in accuracy, even though at that point many adversarial examples are
still visually undetectable for humans. Moreover, the gradient norms remain
much higher on test than on training examples. This suggests that even with PGD
robustification, there are still significant statistical differences between
the network's behavior on the training and test sets.  Given the generality of
our results in terms of architectures, this can perhaps be alleviated only via
tailored architectural constraints on the gradients of the network.  Based on
these theoretical insights, we hypothesize that tweaks on the architecture may
not be sufficient, and coping with the phenomenon of adversarial examples will
require genuinely new ideas.

\subsubsection*{Acknowledgements}
We thank Mart\'in Arjovsky, Ilya Tolstikhin and Diego Fioravanti for helpful
discussions.

\bibliography{advers}
\bibliographystyle{icml2019}

\newpage
\appendix

\normalsize
\section{Proofs\label{proofs}}

\subsection{Proof of Proposition~\ref{lossEqui}}

\begin{proof}
Let $\epsilon \, \dd$ be an adversarial perturbation with $\norm{\dd} = 1$ 
that locally maximizes the loss
increase at point $\xx$, meaning that $\dd = \mathrm{arg\,max}_{\norm{\dd'} \leq
1} \dl \cdot \dd'$. Then, by definition of the dual norm of $\dl$ we
have: $\dl \cdot (\epsilon \dd) = \epsilon \, \dn{\dl}$. Thus
\begin{align*}
    \tilde{\loss}_{\epsilon, \norm{\cdot}}(\xx,c)
        &= \frac 1 2 (\loss(\xx,c) + \loss(\xx+\epsilon \,  \dd, c)) \\
        &= \frac 1 2 (2 \loss(\xx,c) + \epsilon \left | \dl \cdot \dd \right | + o(\norm{\dd})) \\
        &= \loss(\xx,c) + \frac \epsilon 2 \, \dn{\dl} + o(\epsilon) \\
        &= \loss_{\epsilon, \dn{\cdot}}(\xx,c) + o(\epsilon) \, . \qedhere
\end{align*}
\end{proof}

\subsection{Proof of Theorem~\ref{fullcon}\label{proofFullCon}}

\begin{proof}
    Let $x$ designate a generic coordinate of $\xx$. To evaluate the size of $\norm{\dl}_q$,
    we will evaluate the size of the coordinates $\dli$ of $\dl$ by decomposing
    them into
    \begin{equation*}
        \dli = \sum_{k=1}^K \frac{\partial \loss}{\partial f_k} 
                            \frac{\partial f_k}{\partial x}
            =: \sum_{k=1}^K \dlk \, \dfik,
    \end{equation*}
    where $f_k(\xx)$ denotes the logit-probability of $\xx$ belonging to class $k$.
    We now investigate
    the statistical properties of the logit gradients $\dfk$, and then see how
    they shape $\dli$.
    
    \paragraph{Step 1: Statistical properties of $\dfk$.}
    Let $\P(x,k)$ be the set of paths $\pp$ from input neuron $x$ to
    output-logit $k$.
    Let $p-1$ and  $p$ be two successive neurons on path $\pp$, and $\tilde \pp$
    be the same path $\pp$
    but without its input neuron. Let $w_p$ designate the weight from $p-1$ to
    $p$ and $\omega_{\pp}$ be the \emph{path-product}
    $\omega_\pp := \prod_{p \in \tilde \pp} w_p$.
    Finally, let $\sigma_p$ (resp.\ $\sigma_\pp$) be equal to 1 if the
    ReLU of node $p$ (resp.\ if path $\pp$) is active for input $\xx$, and 0
    otherwise.

    As previously noticed by \citet{balduzzi17neural} using the chain rule, we
    see that $\dfik$ is the sum of all $\omega_\pp$ whose path is active, i.e.\
    $
        \dfik(\xx) = \sum_{\pp \in \P(x,k)} \omega_\pp \sigma_\pp .
    $
    Consequently:
    \begin{multline}\label{varEqFull}
        \ee{W, \sigma}{\dfik(\xx)^2}
            = \sum_{\pp \in \P(x,k)} 
                \prod_{p \in \tilde \pp} \ee{W}{w_p^2} \ee{\sigma}{\sigma_p^2} \\
            = |\P(x,k)| \prod_{p \in \tilde \pp} \frac{2}{d_{p-1}} \frac 1 2
            =  \prod_{p \in \tilde \pp} d_p \cdot
               \prod_{p \in \tilde \pp} \frac{1}{d_{p-1}}
            = \frac{1}{d} \, .
    \end{multline}
    The first equality uses \ref{h1} to decouple the expectations over weights
    and ReLUs, and then applies Lemma~\ref{decor} of Appendix~\ref{proofTheo},
    which uses \ref{h2}-\ref{h4} to kill all cross-terms and take the
    expectation over weights inside the product. 
    The second equality uses~\ref{h2} and the fact that the resulting product is
    the same for all active paths. The third equality counts the
    number of paths from $x$ to $k$ and we conclude by noting that all terms cancel out, except
    $d_{p-1}$ from the input layer which is $d$. Equation~\ref{varEqFull} shows that $|\dfik|
    \propto 1/\sqrt d$.

    \paragraph{Step 2: Statistical properties of $\dlk$ and $\dli$.}
    Defining $q_k(\xx) := \frac{e^{f_k(\xx)}}{\sum_{h=1}^K e^{f_h(\xx)}}$ (the
    probability of image $x$ belonging to class $k$ according to the 
    network), we have, by definition of the cross-entropy loss, $\loss(\xx,c)
    := -\log q_c(\xx)$, where $c$ is the label of the target class. Thus:
    \begin{equation*}
        \dlk(\xx) =
            \left \{
            \begin{array}{ll}
                     - q_k(\xx) \quad &\text{if } k \neq c \\
                     1-q_c(\xx) \quad &\text{otherwise},
            \end{array}
            \right .
            \quad \text{and}
    \end{equation*}
    \begin{equation}\label{XentrGradientUgly}
        \dli(\xx) = (1-q_c) \, \dfic(\xx) + \sum_{k \neq c} q_k \, (-\dfik(\xx)).
    \end{equation}
    Using again Lemma~\ref{decor}, we see that  the $\dfik(\xx)$ are $K$ centered
    and uncorrelated variables. So $\dli(\xx)$ is approximately the sum of $K$
    uncorrelated variables with zero-mean, and its total variance is given by
    $\big((1-q_c)^2 + \sum_{k \neq c} q_k^2 \big) / d$.
    Hence the magnitude of $\dli(\xx)$ is $1/\sqrt d$ for all $\xx$, so the
    $\lnorm_q$-norm
    of the full input gradient is $d^{1/q-1/2}$. 
    \eqref{res_dependance} concludes.
\end{proof}

\begin{remark}\label{remOnTraining}
    Equation~\ref{XentrGradientUgly} can be rewritten as
    \begin{equation}\label{XentrGradient}
        \dli(\xx) = \sum_{k=1}^K q_k(\xx) \, \big(\dfic(\xx) - \dfik(\xx)\big) \ . 
    \end{equation}
    As the term $k=c$ disappears, the norm of the gradients $\dli(\xx)$ appears
    to be controlled by the total error probability. 
    This suggests that, even without regularization, trying to decrease the
    ordinary classification error is still a valid strategy against adversarial
    examples. It reflects the fact that when increasing the classification
    margin, larger gradients of the classifier's logits are needed to push
    images from one side of the classification boundary to the other. This is
    confirmed by Theorem~2.1 of \citet{hein17formal}. See also
    \eqref{heinEq} in Appendix~\ref{crosslip}.
\end{remark}

\subsection{Proof of Theorem~\ref{generaltheo}\label{proofTheo}}

The proof of Theorem~\ref{generaltheo} is very similar to the one of
Theorem~\ref{fullcon}, but we will need to first generalize the equalities
appearing in \eqref{varEqFull}. To do so, we identify the computational
graph of a neural network to an abstract Directed Acyclic Graph (DAG) which we
use to prove the needed algebraic equalities. We then concentrate on the
statistical weight-interactions implied by assumption~\H, and finally throw
these results together to prove the theorem. In all the proof, $o$ will
designate one of the output-logits $f_k(\xx)$.
\begin{lemma}\label{dag}
    Let $\xx$ be the vector of inputs to a given DAG, $o$ be any leaf-node of the DAG,
    $x$ a generic coordinate of $\xx$. 
    Let $\pp$ be a path from the set of paths $\P(x,o)$ from $x$ to $o$,
    $\tilde \pp$ the same path without node $x$, $p$ a
    generic node in $\tilde \pp$, and $d_p$ be its input-degree.
    Then:
    \begin{equation}\label{dageq}
        \sum_{x \in \xx} \sum_{\tilde \pp \in \P(x,o)} \prod_{p \in \tilde \pp} \frac{1}{d_p} = 1
    \end{equation}
\end{lemma}

\begin{proof}
    We will reason on a random walk starting at $o$ and going up the
    DAG by choosing any incoming node with equal probability.
    The DAG being finite, this walk will end up at an input-node $x$ with
    probability $1$.
    Each path $\pp$ is taken with probability $\prod_{p \in \tilde \pp} \frac{1}{d_p}$. 
    And the probability to end up at an input-node is the sum of all these
    probabilities, i.e. $\sum_{x \in \xx} \sum_{\pp \in \P(x,o)} \prod_{p \in
    \pp} d_p^{-1}$, which concludes.
\end{proof}

The sum over all inputs $x$ in \eqref{dageq} being 1,
on average it is $1/d$ for each $x$, where $d$ is the total number of inputs
(i.e.\ the length of $\xx$).
It becomes an equality under assumption~\S:

\begin{lemma}\label{symdag}
    Under the symmetry assumption~\S, and with the previous notations, for any
    input $x \in \xx$:
    \begin{equation}
        \sum_{\pp \in \P(x,o)} \prod_{p \in \tilde \pp} \frac{1}{d_p} = \frac{1}{d} \, .
    \end{equation}
\end{lemma}

\begin{proof}
    Let us denote $\Dset := \{d_\pp\}_{x \in \P(x,o)}$.
    Each path $\pp$ in $\P(x,o)$ corresponds to exactly one element $\dset$ in
    $\Dset$ and vice-versa. And the elements $d_p$ of $\dset$ completely
    determine the product $\prod_{p \in \tilde \pp} d_p^{-1}$. By using
    \eqref{dageq} and the fact that, by \S, the multiset $\Dset$ is
    independent of $x$, we hence conclude
    \begin{align*}
        \sum_{x \in \xx} \sum_{\pp \in \P(x,o)} \prod_{p \in \tilde \pp} \frac{1}{d_p} 
            &= \sum_{x \in \xx} \sum_{\dset \in \Dset} \prod_{d_p \in \dset} \frac{1}{d_p} \\
            &= d \sum_{\dset \in \Dset} \prod_{d_p \in \dset} \frac{1}{d_p}
            = 1 \, . \tag*{\qedhere}
    \end{align*}
\end{proof}

Now, let us relate these considerations on graphs to gradients and use
assumptions~\H. We remind that path-product $\omega_\pp$ is the product
$\prod_{p \in \tilde \pp} w_p$.

\begin{lemma}\label{decor}
    Under assumptions \H, the path-products $\omega_{\pp}, \omega_{\pp'}$
    of two distinct paths
    $\pp$ and $\pp'$ starting from a same input node $x$, satisfy:
    \begin{align*}
        \ee{W}{\omega_\pp \, \omega_{\pp'}} = 0 \quad \text{and} \quad 
        \ee{W}{\omega_\pp^2} = \prod_{p \in \tilde \pp} \ee{W}{w_p^2} \, . 
    \end{align*}
    Furthermore, if there is at least one non-average-pooling weight on path
    $\pp$, then $\ee{W}{\omega_{\pp}} = 0$.
\end{lemma}

\begin{proof}
    Hypothesis \ref{h3} yields 
    \begin{equation*}
        \ee{W}{\omega_\pp^2} = \ee{W}{\prod_{p \in \tilde \pp} w_p^2}
            = \prod_{p \in \tilde \pp} \ee{W}{w_p^2} \, .
    \end{equation*}
    Now, take two different paths $\pp$ and $\pp'$ that start
    at a same node $x$. Starting from $x$, consider the first node after which
    $\pp$ and $\pp'$ part and call $p$ and $p'$ the next nodes on $\pp$ and $\pp'$
    respectively. Then the weights $w_p$ and $w_{p'}$ are two weights of a same
    node. Applying \ref{h3} and \ref{h4} hence gives
    \begin{equation*}
        \ee{W}{\omega_\pp \, \omega_{\pp'}} 
            = \ee{W}{\omega_{\pp \backslash p} \, \omega_{\pp' \backslash p'}}
               \ee{W}{w_p \, w_{p'}}
            = 0 \, .
    \end{equation*}
    Finally, if $\pp$ has at least one non-average-pooling node $p$, then
    successively applying \ref{h3} and \ref{h2} yields:
    $\ee{W}{\omega_{\pp}} = \ee{W}{\omega_{\pp \backslash p}}
    \ee{W}{w_p} = 0$.
\end{proof}

We now have all elements to prove Theorem~\ref{generaltheo}.

\begin{proof}(\textbf{of Theorem~\ref{generaltheo}})
    For a given neuron $p$ in $\tilde \pp$, let
    $p-1$ designate the previous node in $\pp$ of $p$. Let 
    $\sigma_p$ (resp.\ $\sigma_\pp$) be a variable equal to $0$ if neuron $p$ gets killed by its
    ReLU (resp.\ path $\pp$ is inactive), and 1 otherwise.
    Then:
    \begin{equation*}
        \partial_x o \
          = \sum_{\pp \in \P(x,o)} \prod_{p \in \tilde \pp} \partial_{p-1} \, p \
          = \sum_{\pp \in \P(x,o)} \omega_\pp \, \sigma_\pp
    \end{equation*}
    Consequently:
    \begin{align}\label{varEq}
        \ee{W, \sigma}{(\partial_x o)^2}
          &= \sum_{\pp, \pp' \in \P(x,o)} \ee{W}{\omega_\pp \, \omega_{\pp'}} 
                \ee{\sigma}{\sigma_\pp \sigma_{\pp'}} \nonumber \\
          &= \sum_{\pp \in \P(x,o)} 
                \prod_{p \in \tilde \pp}
                \ee{W}{\omega_p^2} 
                \ee{\sigma}{\sigma_p^2} \\
          &= \sum_{\pp \in \P(x,o)} \prod_{p \in \tilde \pp} \frac{2}{d_p} \, \frac{1}{2}
          = \frac{1}{d} \, , \nonumber
    \end{align}
    where the first line uses the independence between the ReLU killings and
    the weights (\ref{h1}), the second uses Lemma~\ref{decor} and the last uses
    Lemma~\ref{symdag}.
    The gradient $\bm{\partial_\xx} o$ thus has coordinates whose squared
    expectations scale like $1/d$. Thus each coordinate scales like $1/\sqrt d$
    and $\norm{\bm{\partial_\xx} o}_q$ like $d^{1/2 - 1/q}$. Conclude on
    $\norm{\dl}_q$ and $\epsilon_p \norm{\dl}_q$ by using Step~2 of the proof
    of Theorem~\ref{fullcon}.

    Finally, note that, even without the symmetry assumption \S, using
    Lemma~\ref{dag} shows that
    \begin{align*}
        \ee{W}{\norm{\bm{\partial_\xx} o}_2^2}
          &= \sum_{x \in \xx} \ee{W}{(\partial_x o)^2} \\
          &= \sum_{x \in \xx} \sum_{\pp \in \P(x,o)} \prod_{p \in \tilde \pp} 
                \frac{2}{d_p} \, \frac{1}{2}
          = 1 \, .
    \end{align*}
    Thus, with or without \S, $\norm{\bm{\partial_\xx} o}_2$ is independent of the
    input-dimension $d$.
\end{proof}

\subsection{Proof of Theorem~\ref{avgpooltheo}\label{avgpoolproof}}

To prove Theorem~\ref{avgpooltheo}, we will actually prove the following more
general theorem, which generalizes Theorem~\ref{generaltheo}.
Theorem~\ref{avgpooltheo} is a straightforward corollary of it.

\begin{theorem}
    Consider any feedforward network with linear connections and ReLU activation functions
    that outputs logits $f_k(\xx)$ and satisfies assumptions \H.
    Suppose that there is a fixed multiset of integers $\{a_1, \ldots, a_n\}$ such
    that each path from input to output traverses exactly $n$ average
    pooling nodes with degrees $\{a_1, \ldots, a_n\}$. Then:
    \begin{equation}\label{logitl2norm}
        \norm{\dfik}_2 \propto \frac{1}{\prod_{i=1}^{n} \sqrt{a_i}} \, .
    \end{equation}
    Furthermore, if the net satisfies the symmetry assumption \S, then:
    $
        |\dfik| \propto \frac{1}{\sqrt{d \prod_{i=1}^{n} a_i}} .
    $
\end{theorem}

Two remarks. First, in all this proof, ``weight'' encompasses both the
standard random weights, and the constant (deterministic) weights equal to
$1/(\text{in-degree})$ of the average-poolings. Second,
assumption \ref{h4} implies that the average-pooling nodes have disjoint input
nodes: otherwise, there would be two non-zero deterministic weights $w, w'$
from a same neuron that would hence satisfy: $\ee{W}{w \, w'} \neq 0$.

\begin{proof}
    As previously, let $o$ designate any fixed output-logit $f_k(\xx)$.
    For any path $\pp$, let $\aa$ be the set of average-pooling nodes of $\pp$
    and let $\qq$ be the set of remaining nodes. Each path-product $\omega_\pp$
    satisfies: $\omega_{\pp} = \omega_{\qq} \omega_{\aa}$, where $\omega_{\aa}$
    is a same fixed constant. For two distinct paths $\pp, \pp'$, 
    Lemma~\ref{decor} therefore yields:
    $\ee{W}{\omega_{\pp}^2} = \omega_{\aa}^2 \ee{W}{\omega_{\qq}^2}$ and 
    $\ee{W}{\omega_{\pp} \omega_{\pp'}} = 0$.
    Combining this with Lemma~\ref{symdag} and under assumption~\S, we get
    similarly to \eqref{varEq}:
    \begin{align}\label{varEqPool}
        \ee{W, \sigma}{(\partial_x o)^2}
          &= \! \! \! \! \! \sum_{\pp, \pp' \in \P(x,o)} \! \! \! \! \!
                \omega_{\aa} \omega_{\aa'}
                \ee{W}{\omega_\qq \, \omega_{\qq'}} 
                \ee{\sigma}{\sigma_\qq \sigma_{\qq'}} \nonumber \\
          &= \sum_{\pp \in \P(x,o)} 
                \prod_{i=1}^n \frac{1}{a_i^2} \, 
                \prod_{q \in \tilde \qq}
                \ee{W}{\omega_q^2} 
                \ee{\sigma}{\sigma_q^2} \nonumber \\
          &= \underbrace{
            \prod_{i=1}^n \frac{1}{a_i}
            }_{\substack{\text{same value} \\
                \text{for all } \pp}}
            \smash{\underbrace{
            \sum_{\pp \in \P(x,o)} 
            \underbrace{
            \prod_{i=1}^n \frac{1}{a_i}
            \prod_{q \in \tilde \qq} \frac{2}{d_q} \, \frac{1}{2}
            }_{\prod_{p \in \tilde \pp} \frac{1}{d_p}}
            }_{= \frac 1 d \ \ \text{(Lemma~\ref{symdag})}}} \\
          &= \frac{1}{d} \prod_{i=1}^n \frac{1}{a_i} \, . \nonumber
    \end{align}
    Therefore, $|\partial_x o| = |\dfik| \propto 1/\sqrt{d \prod_{i=1}^n a_i}$.
    Again, note that, even without assumption \S, using
    \eqref{varEqPool} and Lemma~\ref{dag} shows that
    \begin{align*}
        \ee{W}{\norm{\bm{\partial_\xx} o}_2^2}
          &= \sum_{x \in \xx} \ee{W,\sigma}{(\partial_x o)^2} \\
          &\stackrel{\eqref{varEqPool}}{=} \sum_{x \in \xx} \prod_{i=1}^n \frac{1}{a_i}
                \sum_{\pp \in \P(x,o)} 
                \prod_{i=1}^n \frac{1}{a_i}
                \prod_{p \in \tilde \pp} 
                \frac{2}{d_p} \, \frac{1}{2} \\
          &= \prod_{i=1}^n \frac{1}{a_i} \, 
                \underbrace{
                \sum_{x \in \xx} \sum_{\pp \in \P(x,o)}
                \prod_{p \in \tilde \pp} \frac{1}{d_p}
                }_{= 1 \text{ (Lemma~\ref{dag})}}
          = \prod_{i=1}^n \frac{1}{a_i} \, ,
    \end{align*}
    which proves \eqref{logitl2norm}.
\end{proof}

\section{Effects of Strided and Average-Pooling Layers on Adversarial
Vulnerability\label{poolsec}}

It is common practice in CNNs to use average-pooling layers
or strided convolutions to progressively decrease the
number of pixels per channel. 
Corollary~\ref{convnet} shows that using strided convolutions does not protect
against adversarial examples. 
However, what if we replace strided convolutions by convolutions with
stride 1 plus an average-pooling layer?
Theorem~\ref{generaltheo} considers only
\emph{randomly} initialized weights with typical size
$1/\sqrt{\text{in-degree}}$. Average-poolings however
introduce \emph{deterministic} weights of size $1/(\text{in-degree})$.
These are smaller and may therefore
dampen the input-to-output gradients and protect against adversarial examples.
We confirm this in our next theorem, which uses a slightly modified version
($\mathcal H'$) of \H\ to allow average pooling layers. ($\mathcal H'$) is
\H, but where the He-init \ref{h2} applies to all weights \emph{except} the
(deterministic) average pooling weights, and where \ref{h1}
places a ReLU on every non-input \emph{and non-average-pooling} neuron.
\begin{theorem}[\textbf{Effect of Average-Poolings}]\label{avgpooltheo}
    Consider a succession of 
    convolution layers, dense layers and $n$ average-pooling layers, in any
    order, that satisfies ($\mathcal H'$) and outputs logits $f_k(\xx)$.
    Assume the $n$ average pooling layers have  a stride equal to their mask size and
    perform averages over $a_1$, ..., $a_n$ nodes respectively. 
    Then $\norm{\dfk}_2$ and $|\dfik|$ scale like $1/\sqrt{a_1 \cdots a_n}$ and
    $1/\sqrt{d \, a_1 \cdots a_n}$ respectively.
\end{theorem}
Proof in Appendix~\ref{avgpoolproof}.
Theorem~\ref{avgpooltheo} suggest to try and replace any strided convolution by its
non-strided counterpart, followed by an average-pooling layer. 
It also shows that if we systematically reduce the number of pixels per
channel down to 1 by using only non-strided convolutions and
average-pooling layers (i.e.\ $d = \prod_{i=1}^n a_i$), then all input-to-output gradients should become
independent of $d$, thereby making the network completely robust to adversarial
examples.
\begin{figure}[htb]
    \centering
    \vspace{1ex}
    \includegraphics[width=.49 \textwidth]{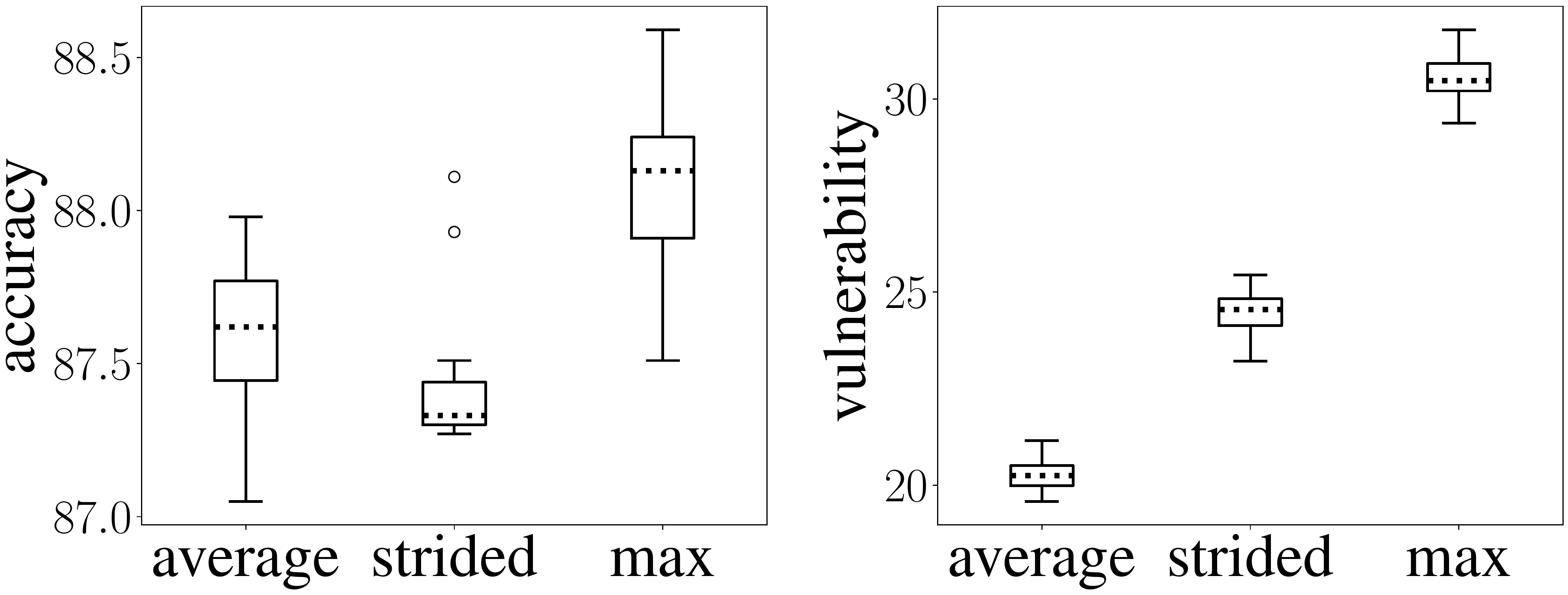}
    \vspace{2ex}
    \caption{\label{fig:pooling_effects}As predicted by
    Theorem~\ref{avgpooltheo}, average-pooling layers make networks more robust
    to adversarial examples, contrary to strided (and max-pooling) ones. But
    the vulnerability with average-poolings remains higher than anticipated.}
\end{figure}
Our following experiments (Figure~\ref{fig:pooling_effects}) show that after training, the networks
get indeed robustified to adversarial examples, but remain more
vulnerable than suggested by Theorem~\ref{avgpooltheo}.

\paragraph{Experimental setup.}
Theorem~\ref{avgpooltheo} shows that, contrary to strided layers, average-poolings
should decrease adversarial vulnerability. We tested this hypothesis on
CNNs trained on CIFAR-10, with 6 blocks of
`convolution $\rightarrow$ BatchNorm $\rightarrow $ReLU' with 64 output-channels,
followed by a final average pooling feeding one neuron per channel to the last
fully-connected linear layer. Additionally, after every second convolution, we
placed a pooling layer with stride and mask-size $(2,2)$ (thus acting on $2 \times 2$ neurons
at a time, without overlap). We tested average-pooling, strided and max-pooling layers
and trained 20 networks per architecture. Results are
shown in Figure~\ref{fig:pooling_effects}. All accuracies are very close,
but, as predicted, the networks with
average pooling layers are more robust to adversarial images than
the others.
However, they remain more vulnerable than what would follow from
Theorem~\ref{avgpooltheo}. We also noticed that, contrary to the strided
architectures, their gradients after training are an
order of magnitude higher than at initialization and than predicted. 
This suggests that assumptions \H\ get more violated when using average-poolings
instead of strided layers.
Understanding why will need further investigations.


\section{Perception Threshold\label{perception}}

To keep the average pixel-wise variation constant across dimensions $d$,
we saw in \eqref{rescale} that the threshold
$\epsilon_p$ of an $\ell_p$-attack should scale like $d^{1/p}$.
We will now see another justification for this scaling.
Contrary to the rest of this work, where we use a fixed $\epsilon_p$ for
all images $\xx$, here we will let $\epsilon_p$ depend on the
$\lnorm_2$-norm of $\xx$. If, as usual, the dataset is normalized such that
the pixels have on average variance 1, both approaches are almost equivalent.

Suppose that given an $\ell_p$-attack norm, we want to 
choose $\epsilon_p$ such that the signal-to-noise ratio (SNR)
$\norm{\xx}_2/\norm{\dd}_2$ of a perturbation $\dd$ with
$\lnorm_p$-norm $\leq \epsilon_p$ is never greater than a given SNR threshold
$1/\epsilon$.
For $p=2$ this imposes $\epsilon_2 = \epsilon \norm{\xx}_2$. More generally,
studying the inclusion of $\lnorm_p$-balls in $\ell_2$-balls yields
\begin{equation}\label{constantSNRp}
\epsilon_p = \epsilon \norm{\xx}_2 d^{1/p-1/2} \, .
\end{equation}
Note that this gives again $\epsilon_p = \epsilon_\infty d^{1/p}$. 
This explains how to adjust the threshold
$\epsilon$ with varying $\ell_p$-attack norm.

Now, let us see how to adjust the
threshold of a given $\ell_p$-norm when the dimension $d$ varies.
Suppose that $\xx$ is a natural image and that decreasing its dimension means
either decreasing its resolution or cropping it. Because the statistics of
natural images are approximately resolution and scale invariant
\citep{huang00statistics}, in either case the average squared value of the image
pixels remains unchanged, which implies that $\norm{\xx}_2$ scales like $\sqrt
d$. Pasting this back into \eqref{constantSNRp}, we again get:
\begin{equation*}
\epsilon_p = \epsilon_\infty \, d^{1/p} \ .
\end{equation*}
In particular,  $\epsilon_\infty \propto \epsilon$ is a dimension free number,
exactly like in \eqref{rescale} of the main part.

Now, why did we choose the SNR as our invariant reference
quantity and not anything else? One reason is that it corresponds to a physical
power ratio between the image and the perturbation, which we think the human
eye is sensible to. Of course, the eye's sensitivity also depends on the
spectral frequency of the signals involved, but we are only interested in
orders of magnitude here.

Another point: any image $\xx$ yields an adversarial perturbation $\dd_{\xx}$,
where by constraint $\norm{\xx}_2 / \norm{\dd_{\xx}} \leq 1/\epsilon$. For
$\ell_2$-attacks, this inequality is actually an equality. 
But what about other $\ell_p$-attacks: (on average over $\xx$,) how far is the
signal-to-noise ratio from its imposed upper bound $1/\epsilon$?
For $p \not \in \{1,2,\infty\}$, the answer unfortunately depends on the
pixel-statistics of the images. But when $p$ is 1 or $\infty$, then
the situation is locally the same as for $p=2$. Specifically:
\begin{lemma}
    Let $\xx$ be a given input and $\epsilon > 0$. Let $\epsilon_p$ be the
    greatest threshold such that for any $\dd$ with $\norm{\dd}_p \leq
    \epsilon_p$, the SNR $\norm{\xx}_2 / \norm{\dd}_2$ is $\leq 1/\epsilon$.
    Then $\epsilon_p = \epsilon \norm{\xx}_2 d^{1/p-1/2}$.
    
    Moreover, for $p \in \{1,2,\infty\}$, if $\dd_{\xx}$ is the
    $\epsilon_p$-sized $\lnorm_p$-attack that locally maximizes the loss-increase
    i.e.\ $\dd_{\xx} = \mathrm{arg\,max}_{\norm{\dd}_p \leq \epsilon_p} |\dl \cdot
    \dd|$, then:
    \begin{equation*}
        \mathop{\mathrm{SNR}}(\xx):= \frac{\norm{\xx}_2}{\norm{\dd_{\xx}}_2}
        = \frac 1 \epsilon
        \ \ \text{and} \ \
        \ee{\xx}{\mathop{\mathrm{SNR}}(\xx)} = \frac 1 \epsilon \ .
    \end{equation*}
\end{lemma}
\begin{proof}
    The first paragraph follows from the fact that the greatest $\lnorm_p$-ball
    included in an $\ell_2$-ball of radius $\epsilon \norm{\xx}_2$ has radius
    $\epsilon \norm{\xx}_2 d^{1/p - 1/2}$. 
    
    The second paragraph is clear for $p=2$. For $p=\infty$, it follows from
    the fact that $\dd_{\xx} = \epsilon_\infty \sign{\dl}$ which satisfies:
    $\norm{\dd_{\xx}}_2 = \epsilon_\infty \sqrt d = \epsilon \norm{\xx}_2$.
    For $p=1$, it is because $\dd_{\xx} =
    \epsilon_1 \max_{i=1..d} |(\dl)_i|$, which satisfies:
    $\norm{\dd_{\xx}}_2 = \epsilon_2 / \sqrt d = \epsilon \norm{\xx}_2$.
\end{proof}
Intuitively, this means that for $p \in \{1,2,\infty\}$,
the SNR of $\epsilon_p$-sized $\lnorm_p$-attacks on any input $\xx$ will be
exactly equal to its fixed upper limit $1/\epsilon$. And in particular, the
mean SNR over samples $\xx$ is the same ($1/\epsilon$) in all three cases.

\section{Comparison to the Cross-Lipschitz Regularizer\label{crosslip}}

In their Theorem~2.1, \citet{hein17formal} show that the minimal
$\epsilon=\norm{\dd}_p$ perturbation to fool the classifier must be bigger
than:
\begin{equation}\label{heinEq}
    \min_{k \neq c} \frac{f_c(\xx) - f_k(\xx)}
                         {\max_{y \in B(x, \epsilon)}
                          \norm{\dfc(y) - \dfk(y)}_q} \, .
\end{equation}
They argue that the training procedure typically already tries to maximize
$f_c(\xx) - f_k(\xx)$, thus one only needs to additionally ensure that
$\norm{\dfc(\xx) - \dfk(\xx)}_q$ is small.
They then introduce what they call a Cross-Lipschitz Regularization, which
corresponds to the case $p=2$ and involves the gradient differences between
\emph{all} classes:
\begin{equation}\label{heinregu}
    \mathcal{R}_{\mathrm{xLip}} :=
        \frac{1}{K^2} \sum_{k,h=1}^K \norm{\dfh(\xx) - \dfk(\xx)}_2^2
\end{equation}
In contrast, using \eqref{XentrGradient}, (the square of) our proposed regularizer
$\norm{\dl}_q$ from \eqref{newloss} can be rewritten, for $p=q=2$ as:
\begin{multline}\label{ourregu}
    \mathcal R_{\norm{\cdot}_2}(f) 
          = \sum_{k,h=1}^K q_k(\xx) q_h(\xx) \,
                \big(\dfc(\xx) - \dfk(\xx)\big) \cdot \\
                \cdot \big(\dfc(\xx) - \dfh(\xx)\big)
\end{multline}
Although both \eqref{heinregu} and \eqref{ourregu} consist in $K^2$ terms,
corresponding to the $K^2$ cross-interaction between the $K$ classes, the big
difference is that while in \eqref{heinregu} all classes play exactly the same
role, in \eqref{ourregu} the summands all refer to the target class $c$ in at
least two different ways.  First, all gradient differences are always taken
with respect to $\dfc$. Second, each summand is weighted by the probabilities
$q_k(\xx)$ and $q_h(\xx)$ of the two involved classes, meaning that only the
classes with a non-negligible probability get their gradient regularized. This
reflects the idea that only points near the margin need a gradient
regularization, which incidentally will make the margin sharper.

\section{A Variant of Adversarially-Augmented Training\label{fgsmVariant}}

In usual adversarially augmented training, the adversarial image $\xx+\dd$ is
generated on the fly, but is nevertheless treated as a fixed input of the
neural net, which means that the gradient does not get
backpropagated through $\dd$. This need not be. As $\dd$ is itself a
function of $\xx$, the gradients could actually also be backpropagated through
$\dd$. As it was only a one-line change of our code, we used this opportunity
to test this variant of adversarial training (FGSM-variant in
Figure~\ref{fig:results}). 
But except for an increased computation time, we
found no significant difference compared to usual augmented training.

\section{Additional Figures on the Experiments of Section~\ref{penResults}}

\paragraph{Effect of Changing the Attack-Method on Adversarial Vulnerability}
To verify that our empirical results on adversarial vulnerability were
essentially unaffected by the attack method, we measured the adversarial
vulnerability of each network trained in Section~\ref{penResults} using, not
only PGD-attacks (as shown in the figures of the main text), but various other
attack-methods. We tested single-step $\ell_\infty$- (FGSM) and
$\ell_2$-attacks, iterative $\ell_\infty$- (PGD without random start) and
$\ell_2$-attacks, and DeepFool attacks \citep{moosavi16deepfool}. 

Figure~\ref{fig:normDependence} illustrates the results. While
Figure~\ref{fig:results} from the main part fixed the attack type --~iterative
$\lnorm_\infty$-attacks~-- and plotted the curves obtained for various training
methods, Figure~\ref{fig:normDependence} now fixes the training method
--~gradient $\lnorm_1$-regularization~-- and plots the obtained adversarial
vulnerabilities for the different attack types. Figure~\ref{fig:normDependence}
shows that, while the adversarial vulnerability values vary considerably from
method to method, the overall relation between gradient-norms or
regularization-strengths on the one side and vulnerability on the other is
extremely similar for all methods: it increases almost linearly with increasing
gradient-norms and decreasing regularization-strength. Changing the
attack-method in Figure~\ref{fig:results} (main part) hence essentially
changes only the vulnerability scale, not the shape of the curves. Moreover,
the functional-like link between average gradient-norms and every single
approximation of adversarial vulnerability confirms that the first-order
vulnerability is an essential component of adversarial vulnerability.

\begin{figure}[htb]
    \centering
    \includegraphics[width=\linewidth]{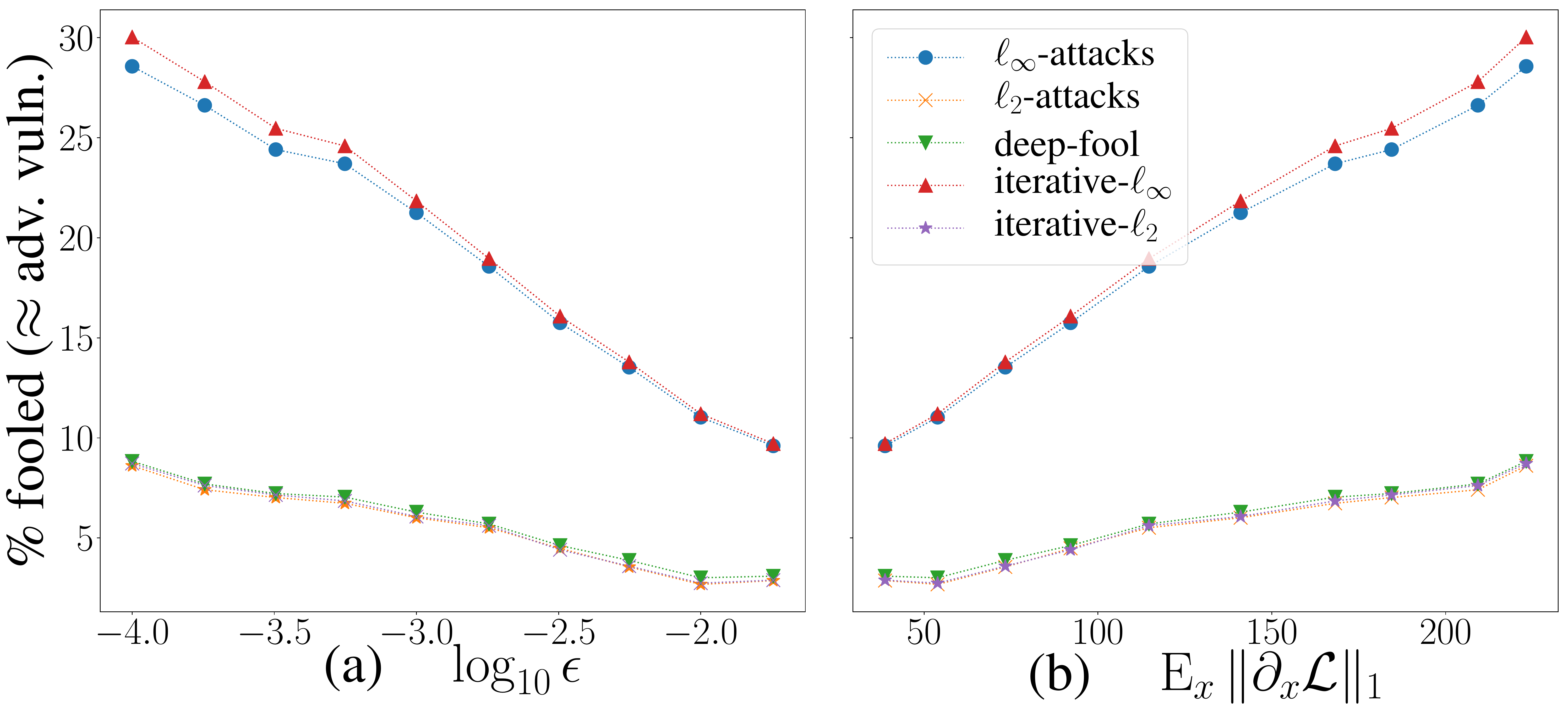}
    \caption{Adversarial vulnerability approximated
    by different attack-types for 10 trained networks as a function of $(a)$ the
    $\lnorm_1$ gradient regularization-strength $\epsilon$ used to train the
    nets and $(b)$ the average gradient-norm. These curves confirm that the
    first-order expansion term in \eqref{dualnorm} is a crucial component of
    adversarial vulnerability.
    \label{fig:normDependence}}
    \vspace{-1ex}
\end{figure}

Note that on Figure~\ref{fig:normDependence}, the two $\lnorm_\infty$-attacks
seem more efficient than the others. This is because we bounded the
attack threshold $\epsilon_\infty$ in $\lnorm_\infty$-norm, whereas the
$\ell_2$- (single-step and iterative) and DeepFool attacks try to minimize the
$\ell_2$-perturbation. With an $\lnorm_2$-threshold, we get the
opposite: which brings us to Figure~\ref{fig:normDependence2}.

\begin{figure*}[!htb]
    \centering
    \includegraphics[width=\linewidth]{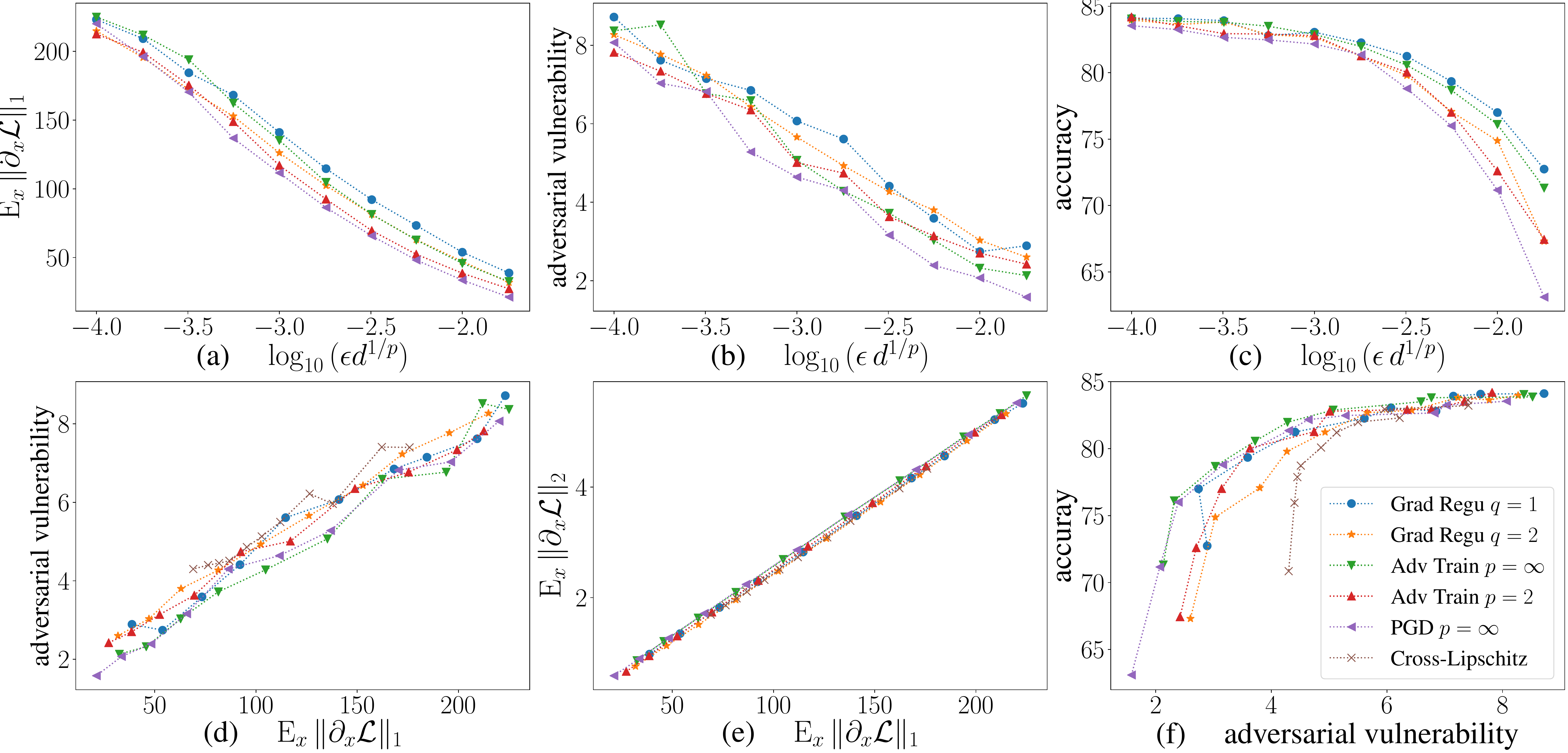}
    \caption{Same as Figure~\ref{fig:results}, but with an $\lnorm_2$-
    perturbation-threshold (instead of $\lnorm_\infty$) and deep-fool attacks
    \citep{moosavi16deepfool} instead of iterative $\lnorm_\infty$ ones.
    All curves look essentially the same than in Fig.~\ref{fig:results}.
    \label{fig:results2}
    }
    \vspace{-2ex}
\end{figure*}

\paragraph{Figures with an $\lnorm_2$ Perturbation-Threshold and
Deep-Fool Attacks}
Here we plot the same curves than on Figures~\ref{fig:results} (main part)
and~\ref{fig:normDependence}, but using an $\ell_2$-attack threshold of size
$\epsilon_2 = \epsilon_\infty \sqrt d$ instead of the $\lnorm_\infty$-threshold, and, for
Fig.~\ref{fig:results2}, using deep-fool attacks \citep{moosavi16deepfool}
instead of iterative $\ell_\infty$-ones. 
Note that contrary to $\lnorm_\infty$-thresholds, $\lnorm_2$-thresholds must be
rescaled by $\sqrt d$ to stay consistent across dimensions (see
Eq.\ref{rescale} and Appendix~\ref{perception}). All curves look essentially
the same as their counterparts with an $\lnorm_\infty$-threshold.
\begin{figure}[htb]
    \centering
    \includegraphics[width=\linewidth]{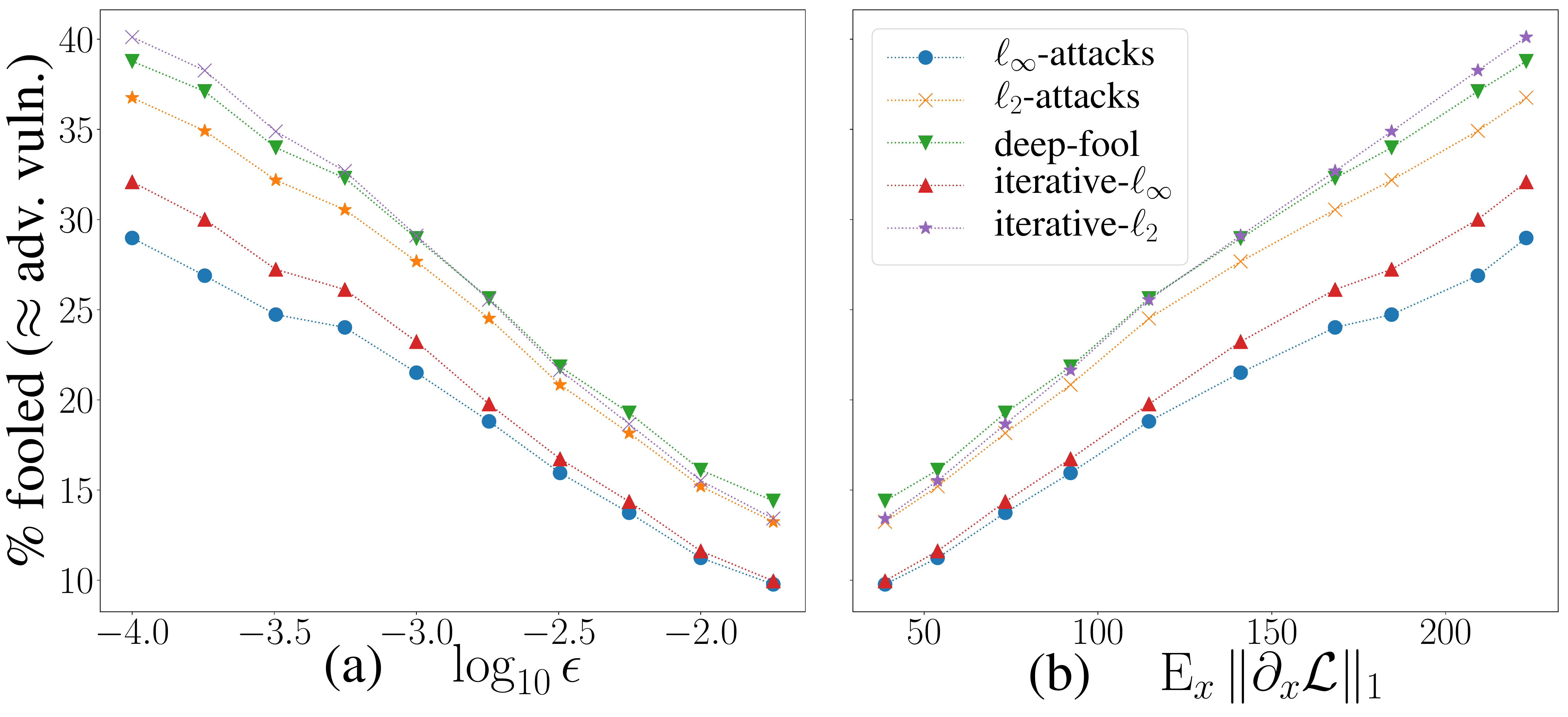}
    \caption{Same as Figure~\ref{fig:normDependence}
    but using an $\lnorm_2$ threshold instead of a $\lnorm_\infty$ one. Now the
    $\lnorm_2$-based methods (deep-fool, and single-step and iterative
    $\lnorm_2$-attacks) seem more effective than the $\lnorm_\infty$ ones.
    \label{fig:normDependence2}}
    \vspace{-2ex}
\end{figure}

\newpage
\onecolumn
\section{Additional Material on the Up-Sampled CIFAR-10 Experiments of Section~\ref{resoDependance}\label{sec:additional_figures}}

\subsection{Network architectures\label{sec:architecture}}

The architecture of the CNNs used in Section~\ref{resoDependance} was a
succession of 8 `convolution $\rightarrow$ batchnorm $\rightarrow$ ReLU' layers
with $64$ output channels, followed by a final full-connection to the
logit-outputs. We used $2\times2$-max-poolings after the convolutions of layers
2,4, 6 and 8, and a final max-pooling after layer 8 that fed only 1 neuron per
channel to the fully-connected layer. To ensure that the convolution-kernels
cover similar ranges of the images across each of the 32, 64, 128 and 256
input-resolutions, we respectively dilated all convolutions (`\`a trous') by a
factor $1$, $2$, $4$ and $8$.

\subsection{Additional Plots}
Here we provide various additional plots computed with the networks trained in
Section~\ref{resoDependance} on upsampled CIFAR-10 images of various sizes.

We first reproduce on Figure~\ref{fig:bcifar_all} the equivalent of
Figure~\ref{fig:results} (that compared the different regularization methods)
but with each curve now representing a specific input-size instead of a
regularization method. Figure~\ref{fig:bcifar_te_training_curves} then analyses
the evolution over training epochs of the test set performances on the
up-sampled 3x256x256 CIFAR-10 images and unveils a striking discrepancy between
error-rate (-damage) and cross-entropy loss (-damage): the cross-entropy
clearly overfits, but the error-rate does not. This motivates a small
comparison between performance at end-of-training and at early stopping (i.e.\
at the epochs with minimal cross-entropy loss).
Figure~\ref{fig:bcifar_early_vs_last} therefore merges several plots from the
training curves of Figure~\ref{fig:bcifar_te_training_curves} by using the
epochs an implicit parameter, and compares their relation at end-of-training
and after early-stopping. Figure~\ref{fig:bcifar_early_vs_last2} continues the
comparison between end-of-training and early-stopping.
Figure~\ref{fig:bcifar_tr_training_curves} then essentially plots the
equivalent of Fig.\ref{fig:bcifar_te_training_curves} but for the training-set
values, showing that, contrary to the test set values, the training-error and
-loss and -loss-damage decrease over training. This adversarial loss-damage
appears to be much smaller on the training than on the test set, which
motivates our last figure, Figure~\ref{fig:tr_te_norms}, that compares the
training and test gradient $\lnorm_1$-norms for all input resolutions. It
confirms the huge discrepancy between the gradient norms on the training and
test set. This suggests that, outside the training sample, and without strong
regularization, the networks tend to recover their prior gradient-properties,
i.e.\ naturally large gradients.

For detailed comments, see figures' captions. Note that, for improved
readability, Figures \ref{fig:bcifar_te_training_curves},
\ref{fig:bcifar_tr_training_curves} \& \ref{fig:tr_te_norms} were smoothed
using an exponential moving average with weight $0.9$, $0.6$ and $0.6$
respectively (higher weights $\rightarrow$ smoother).

\begin{figure*}[!htb]
    \centering 
    \includegraphics[width=\linewidth]{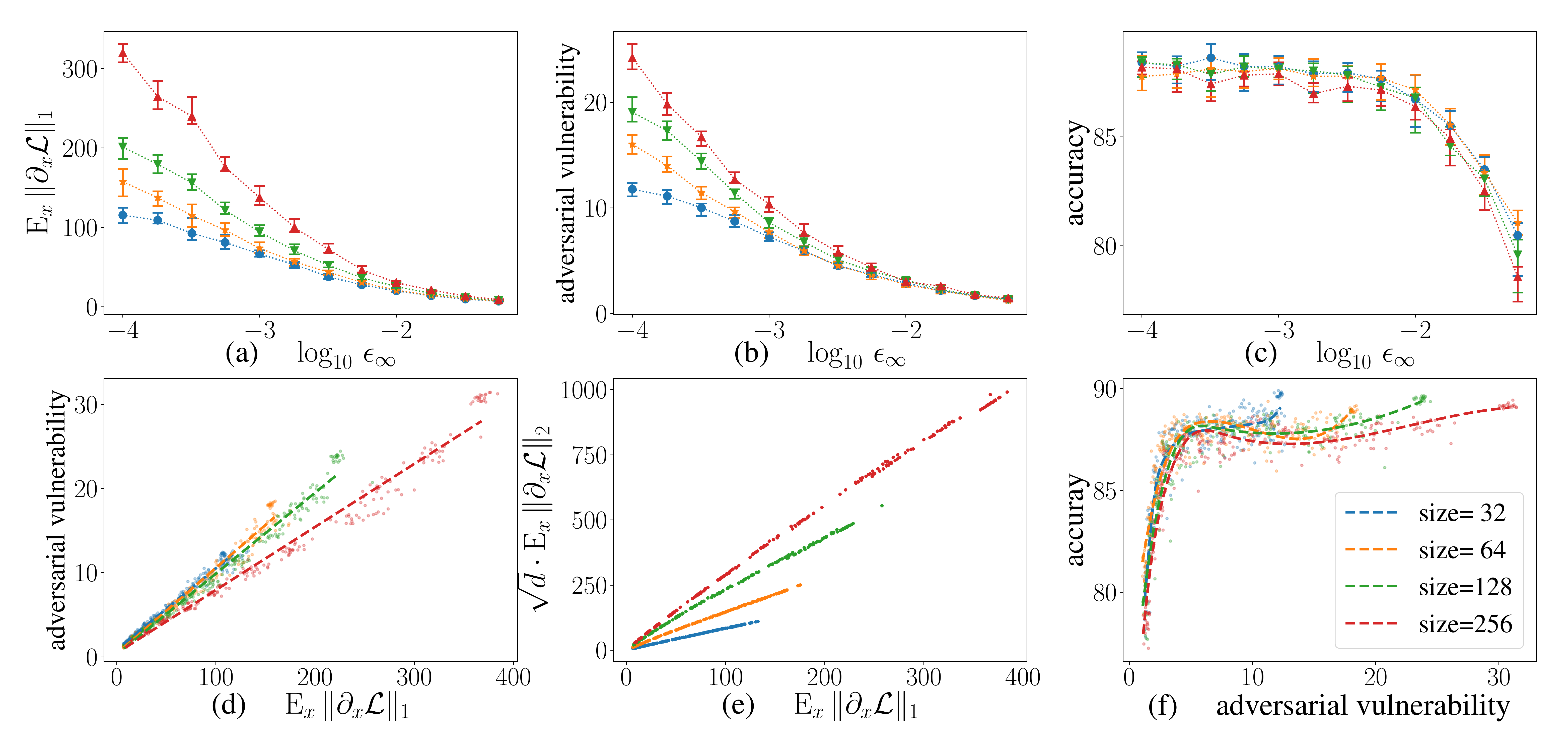}
    \caption{\label{fig:bcifar_all}Equivalent of Figure~\ref{fig:results}, but
    with each curve representing one specific input-size (using up-sampled
    CIFAR-10 images) rather than one training method.  Recall that all values
    (gradient-norms, vulnerability and accuracy) were measured over the last 20
    training epochs on the test set.  They appear all as an individual points
    on the bottom-row plots, and are summarized with errorbars on the
    upper-row.  (d): confirms the functional- (linear-) like relation between
    average loss-gradient norms and adversarial vulnerability. While the slope
    of this relation stays unchanged for images of height and width $\leq 128$,
    it gets slightly dampened for size $256$. Overall, this plot confirms that
    first-order vulnerability (i.e.\ gradient-norms) is an essential part of
    adversarial vulnerability.  (e): confirms the linear relationship between
    $\lnorm_1$- and $\lnorm_2$-gradient-norms (which explains why protecting
    against $\lnorm_\infty$-attacks also protects against $\lnorm_2$-attacks
    and vice-versa), but reveals that the slope does not just change like
    $\sqrt d$ with growing dimension.  Figs.\fga \& \fgb are the same than
    Figs.\ffb \& \ffa (main part), but with a different presentation. Figs.\fgc
    \& \fgf are the same than Figs.\ffc \& \ffd (see main part for comments).}
\end{figure*}

\begin{figure*}[!htb]
    \centering 
    \includegraphics[width=\linewidth]{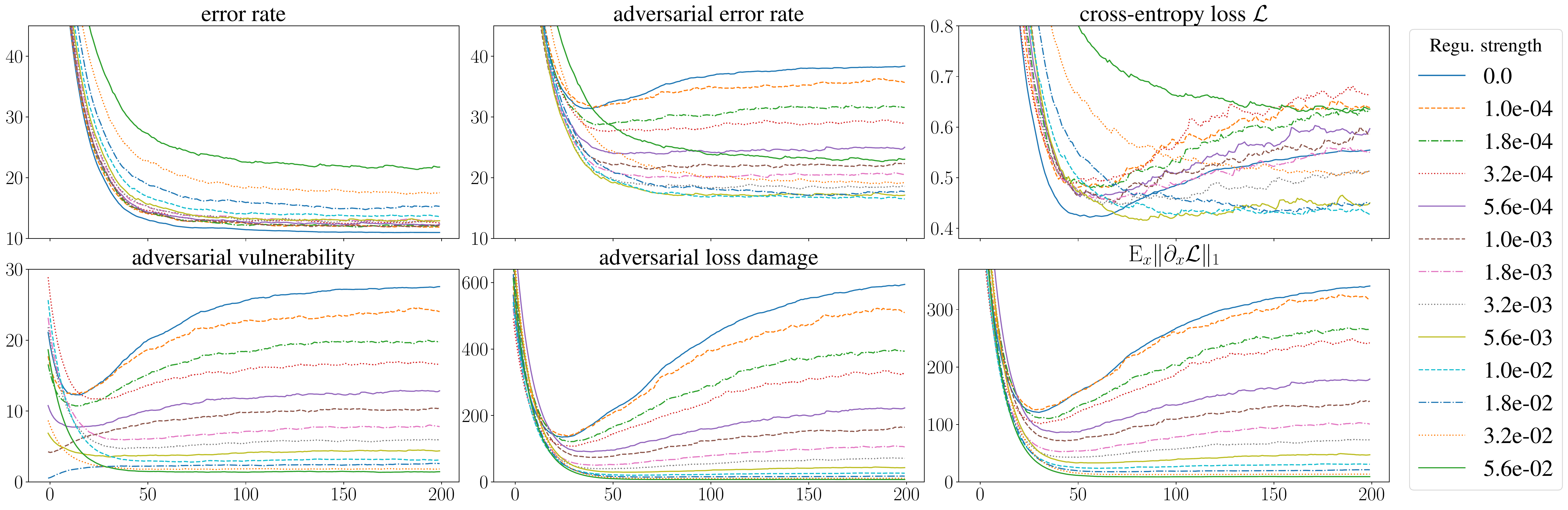}
    \caption{\label{fig:bcifar_te_training_curves}Evolution over the training
    epochs of the networks' test-set performances on the 3x256x256 up-sampled
    CIFAR-10 dataset. We call ``adversarial error-rate'' the error-rate after
    attack (i.e.\ usual error-rate plus accuracy-damage). We also divided the
    adversarial loss damage by the attack-threshold $\epsilon_\infty$ to
    get the same units than $\ee{\xx}{\dl}$ (see
    Fig.\ref{fig:bcifar_early_vs_last} for explanations). While the usual
    error-rate constantly decreases on the test-set (hence showing no sign of
    overfitting), surprisingly, with low or no PGD-regularization, the
    cross-entropy loss (i.e.\ the training objective) severely increases after
    approximately 50 epochs. Moreover, the adversarial error-rate (and
    vulnerability/loss-damage/gradient-norms) curve has a strikingly similar
    shape. Hence, even though the accuracy improves, the cross-entropy
    overfitting still signals some form of overfitting (that could be called
    `gradient-overfitting'), which makes the network more vulnerable.}
\end{figure*}

\begin{figure*}[tb]
    \centering 
    \includegraphics[width=\linewidth]{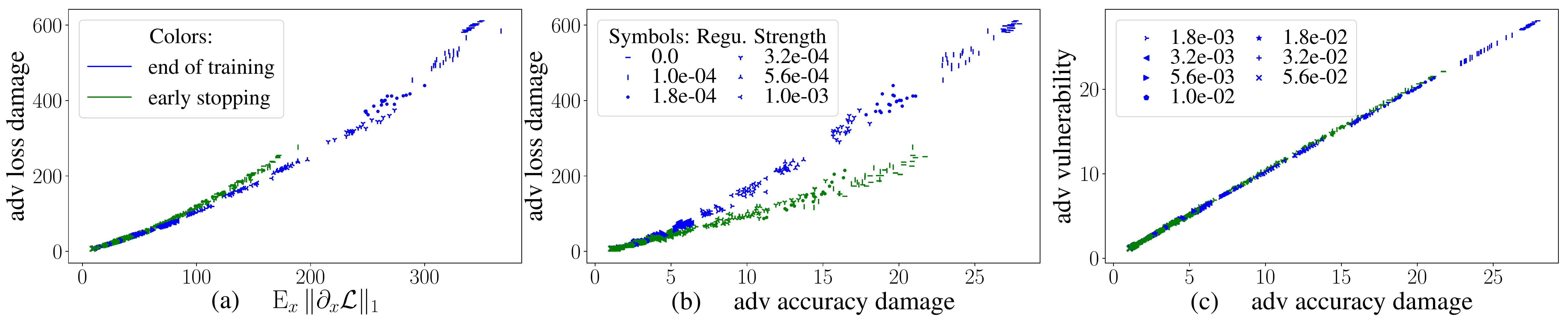}
    \caption{\label{fig:bcifar_early_vs_last}Relations between adversarial
    vulnerability, loss- and accuracy-damage, and loss-gradient norms computed
    on the up-sampled 3x256x256 CIFAR-10 test set images over the 20 last
    training epochs (blue) and at the 20 optimal early-stopping epochs (green),
    i.e.\ the 20 epochs with smallest cross-entropy test-loss. Note that these
    plots essentially merge different plots from
    Fig.\ref{fig:bcifar_te_training_curves} by using the 20 end-of-training and
    20 early-stopping epochs as a common implicit parameter.  As in
    Fig.\ref{fig:bcifar_te_training_curves}, we divided the adversarial loss
    damage by the attack-threshold $\epsilon_\infty$. This `normalized' loss
    damage can thus be understood as the average loss-gradient norm between an
    image $\xx$ and its adversarial perturbation $\xx+\dd$, and can directly be
    compared with $\ee{\xx}{\dl}$. (a) Gradient norms appear to be a stable
    indicator for loss-damage through-out training. Note however that the
    gradient norms at the original input points are on average only half the
    size of the gradients of their surroundings. That might explain why in
    practice, iterative gradient regularization (e.g.\ PGD) is more effective
    than single-step regularization (e.g.\ FGSM). (b) Adversarial accuracy- and
    loss-damage are in a functional-like relationship, but which evolves over
    training (thus the difference between end-of-training and early-stopping).
    (c) Adversarial vulnerability and accuracy-damage are in a constant, almost
    perfectly proportional relationship. Comparing (b) and (c) suggests that
    the main difference between adversarial loss-damage and adversarial
    vulnerability comes from the difference between the $\loss_{0/1}$- and the
    cross-entropy loss $\loss$.}
\end{figure*}

\begin{figure*}[!htb]
    \centering 
    \includegraphics[width=\linewidth]{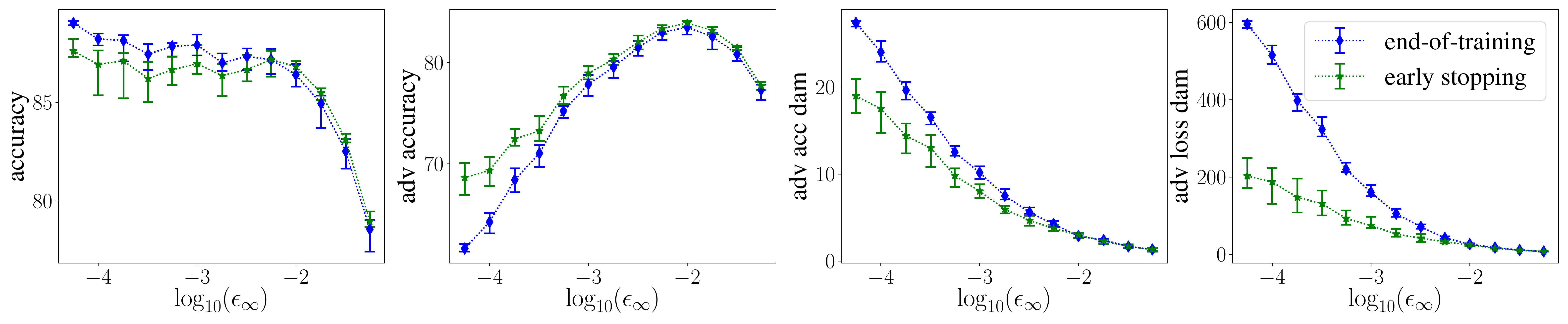}
    \caption{\label{fig:bcifar_early_vs_last2} Network performance at early
    stopping versus end-of-training, for different regularization strength, on
    up-sampled 3x256x256 CIFAR-10 images. Training past the epochs with minimal
    cross-entropy test-loss might improve the final test-accuracy, but
    significantly increases the networks' vulnerability. (The left-most point
    of each curve actually corresponds to $\epsilon_\infty=0.$).}
\end{figure*}

\begin{figure*}[!htb]
    \centering 
    \includegraphics[width=\linewidth]{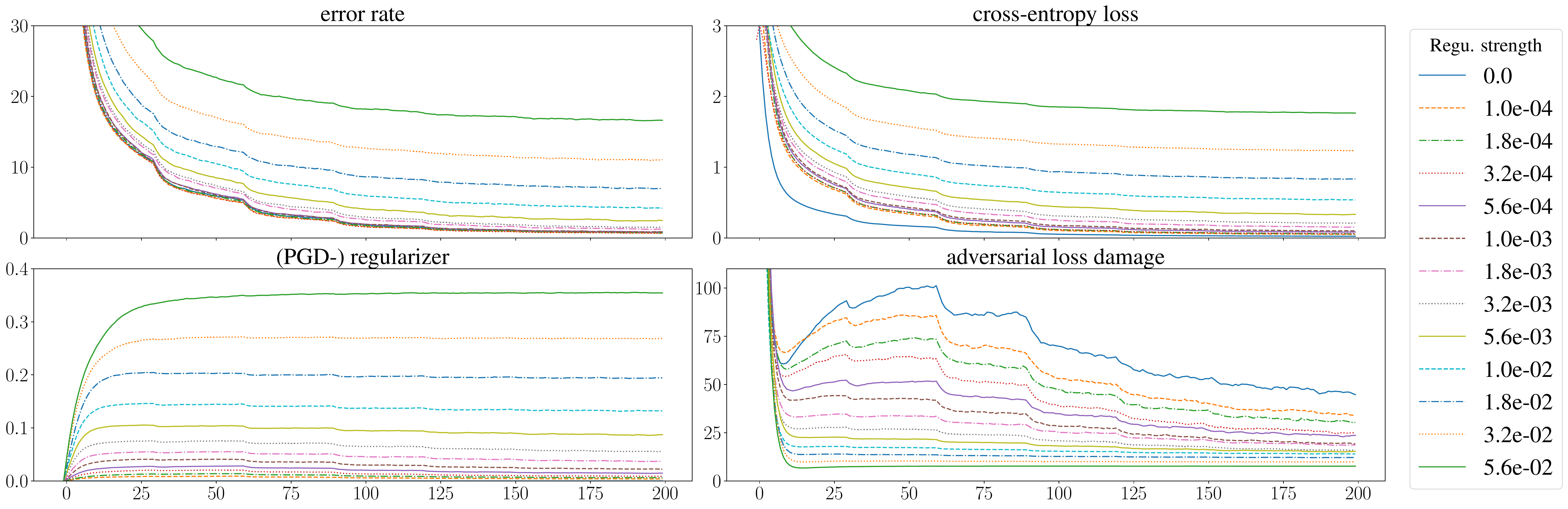}
    \caption{\label{fig:bcifar_tr_training_curves}
    Evolution over the training epochs of the networks' \emph{training-set}
    performances on the 3x256x256 up-sampled CIFAR-10 dataset. Compare with
    Fig.\ref{fig:bcifar_te_training_curves} for the corresponding
    \emph{test-set} performances. Contrary to the test-set performances,
    error-rate, cross-entropy and adversarial loss-damage all steadily decrease
    (after some initialization epochs).}
\end{figure*}

\begin{figure*}[!htb]
    \centering 
    \includegraphics[width=\linewidth]{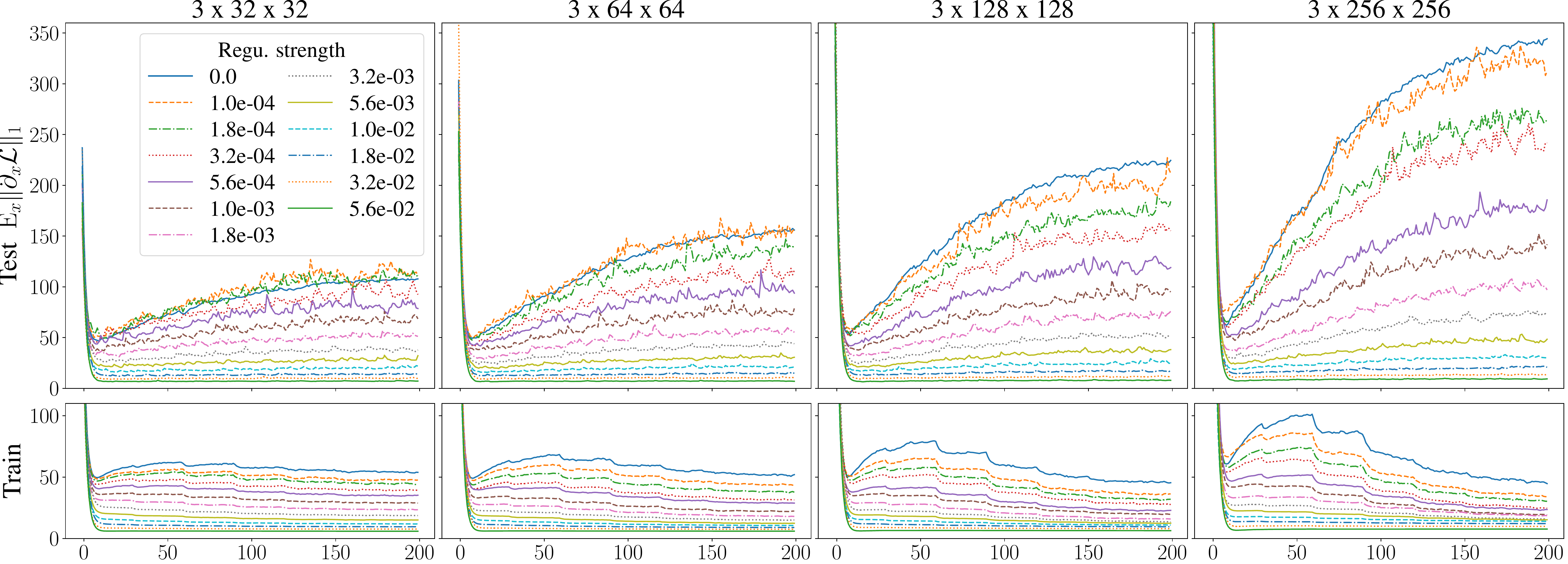}
    \caption{\label{fig:tr_te_norms}Evolution over training epochs of the
    average $\lnorm_1$-gradient-norms on the test (top row) and training set
    (bottom row). \emph{There is a clear discrepancy between training and test
    set values}: after around 50-epochs of initialization, the gradient norms
    constantly decrease on the training set and become dimension independent
    (even without regularization); on the test set however, they increase and
    scale like $\sqrt{d}$. This suggests that, outside the training points, and
    without very strong gradient regularization, the nets tend to recover their
    prior gradient-properties (i.e.\ naturally large gradients).}
\end{figure*}

\newpage
\section{Figures for the Experiments of Section~\ref{resoDependance} on the
Custom Mini-ImageNet Dataset \label{sec:imgnet}}

\begin{figure*}[!htb]
    \centering 
    \includegraphics[width=\linewidth]{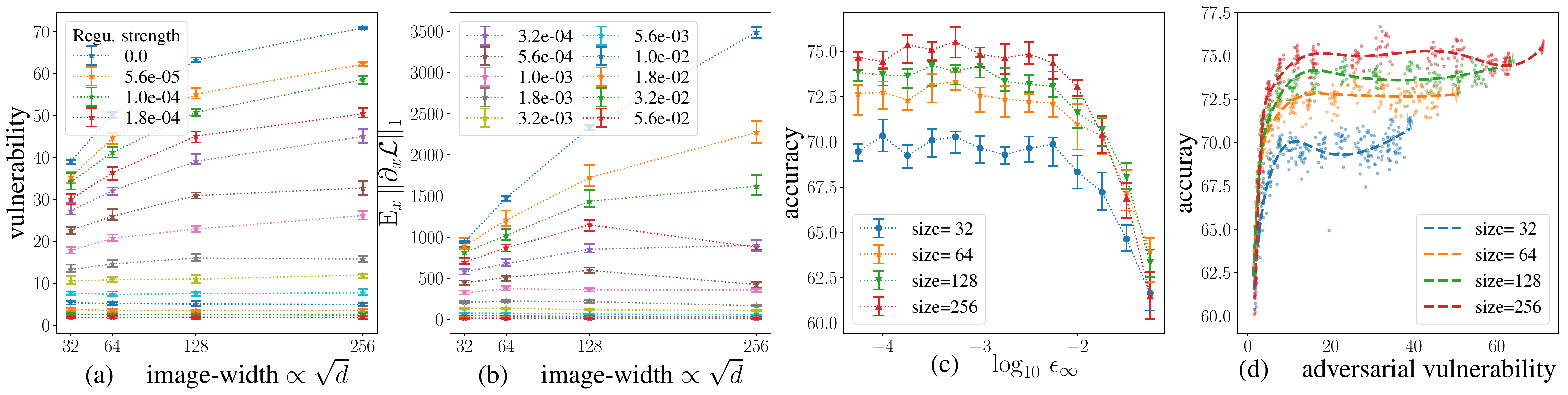}
    \caption{\label{fig:imgnet_summary}Same as
    Figure~\ref{fig:bcifar_summary}, but using down-sampled images from our
    custom 12-class `Mini-ImageNet' dataset (see Sec.\ref{resoDependance}) rather than up-sampled CIFAR-10
    images. Interestingly, (d) shows that PGD training finds better
    accuracy-vulnerability trade-offs with higher input dimensions. It is thus
    more effective at tackling adversarial vulnerability than a simple initial
    down-sampling layer that would be used to reduce the data's dimension.}
\end{figure*}

\begin{figure*}[!htb]
    \centering 
    \includegraphics[width=\linewidth]{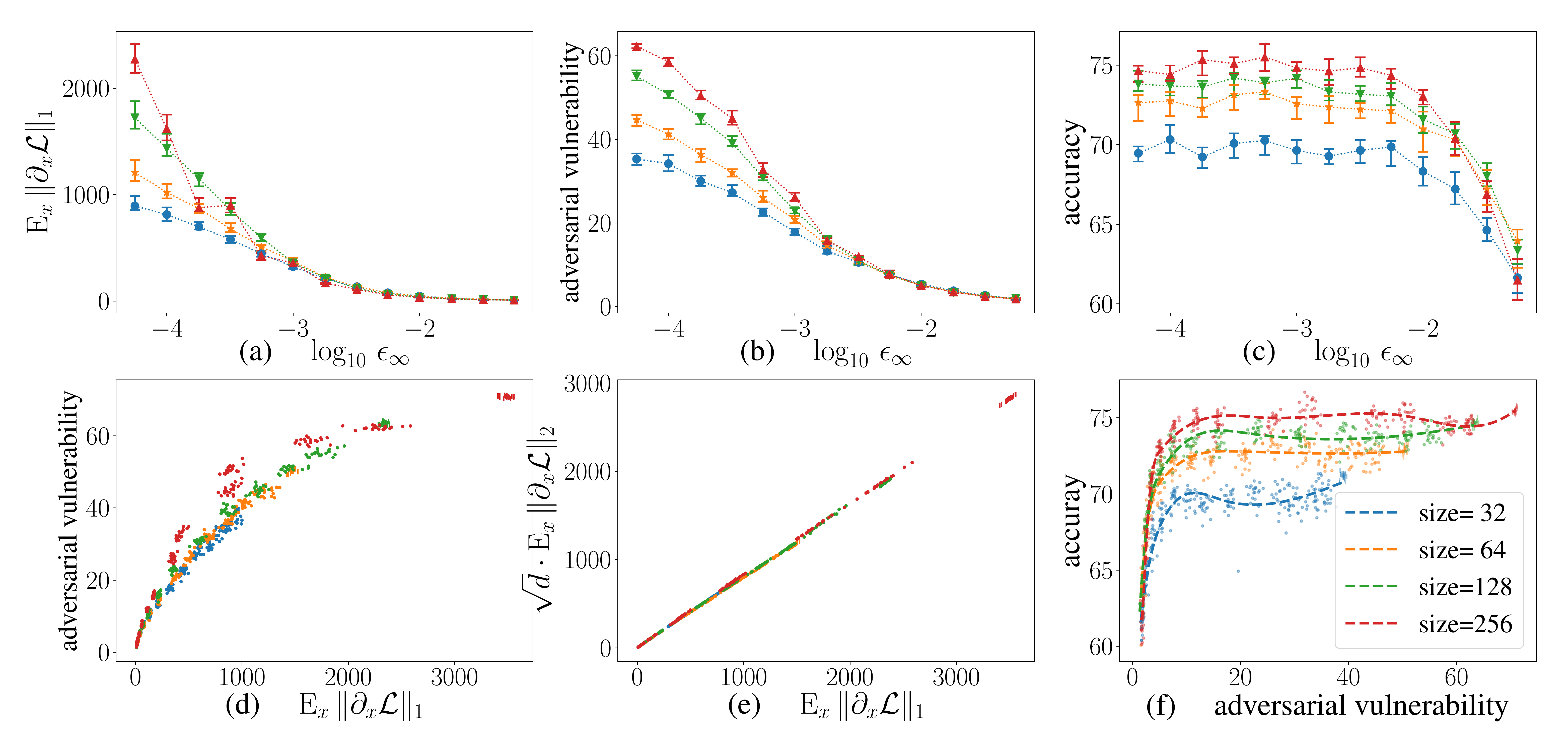}
    \caption{\label{fig:imgnet_all}Same as
    Figure~\ref{fig:bcifar_all}, but using down-sampled images from our custom
    12-class `Mini-ImageNet' dataset rather than up-sampled CIFAR-10 images.}
\end{figure*}

\end{document}